\documentclass[11pt, ruled, linesnumbered, vlined]{article}%

\usepackage{amsfonts}
\usepackage{color}
\usepackage[usenames, dvipsnames, svgnames, table]{xcolor}
\usepackage{fullpage}
\usepackage{amsmath}
\usepackage{amsthm}
\usepackage{comment}
\usepackage{graphicx}
\usepackage[round]{natbib}

\usepackage{amssymb}
\usepackage{authblk}
\usepackage[font=footnotesize, labelfont=bf]{caption}
\usepackage{algorithm2e}
\usepackage{multirow}
\usepackage{subcaption}
\usepackage{booktabs}

\usepackage{enumerate}
\usepackage[shortlabels]{enumitem}

\definecolor{darkred}{rgb}{0.7,0,0}
\definecolor{darkblue}{rgb}{0,0,0.7}
\definecolor{darkgreen}{rgb}{0,0.4,0}

\usepackage[backref, colorlinks=true]{hyperref}
\hypersetup{
	linktocpage={true},
	linkcolor = {darkred},
	citecolor = {darkblue} 
}

\newtheorem{thm}{Theorem}
\newtheorem{prop}{Proposition}
\newtheorem{asmp}{Assumption}

\newtheorem{defn}{Definition}
\newtheorem{lem}{Lemma}
\newtheorem{rem}{Remark}

\theoremstyle{definition}
\newtheorem{ex}{Example}
\makeatother

\begin{document}

\title{Bayesian Frequency Estimation Under Local Differential Privacy With an Adaptive Randomized Response Mechanism}

\author[]{Soner Aydın}
\author[]{Sinan Yıldırım}
\affil[]{Sabancı University, Faculty of Engineering and Natural Sciences, İstanbul, Turkey}
\affil[]{\small{\texttt{soner.aydin@sabanciuniv.edu, sinanyildirim@sabanciuniv.edu}}}

\date{\today}
\maketitle


\begin{abstract}
Frequency estimation plays a critical role in many applications involving personal and private categorical data. Such data are often collected sequentially over time, making it valuable to estimate their distribution online while preserving privacy. We propose AdOBEst-LDP, a new algorithm for adaptive, online Bayesian estimation of categorical distributions under local differential privacy (LDP). The key idea behind AdOBEst-LDP is to enhance the utility of future privatized categorical data by leveraging inference from previously collected privatized data. To achieve this, AdOBEst-LDP uses a new adaptive LDP mechanism to collect privatized data. This LDP mechanism constrains its output to a \emph{subset} of categories that `predicts' the next user's data. By adapting the subset selection process to the past privatized data via Bayesian estimation, the algorithm improves the utility of future privatized data. To quantify utility, we explore various well-known information metrics, including (but not limited to) the Fisher information matrix, total variation distance, and information entropy. For Bayesian estimation, we utilize \emph{posterior sampling} through stochastic gradient Langevin dynamics, a computationally efficient approximate Markov chain Monte Carlo (MCMC) method.

We provide a theoretical analysis showing that (i) the posterior distribution of the category probabilities targeted with Bayesian estimation converges to the true probabilities even for approximate posterior sampling, and (ii) AdOBEst-LDP eventually selects the optimal subset for its LDP mechanism with high probability if posterior sampling is performed exactly. We also present numerical results to validate the estimation accuracy of AdOBEst-LDP. Our comparisons show its superior performance against non-adaptive and semi-adaptive competitors across different privacy levels and distributional parameters.

\textbf{Keywords:} Local differential privacy, frequency distribution, adaptive estimation, posterior sampling, stochastic gradient Langevin dynamics 
\end{abstract}

\section{Introduction} \label{section:Intro}
Frequency estimation is the focus of many applications that involve personal and private categorical data. Suppose a type of sensitive information is represented as a random variable $X$ with a categorical distribution denoted by $\text{Cat}(\theta)$, where $\theta$ is a $K$-dimensional probability vector. As real-life examples, this could be the distribution of the types of a product bought by the customers of an online shopping company, responses to a poll question like ``Which party will you vote for in the next elections?'', occupational affiliations of the people who visit the website of a governmental agency, and so on.

In this paper, we propose an adaptive and online algorithm to estimate $\theta$ in a \textit{Local Differential Privacy} (LDP) framework where $X$ is unobserved and instead, we have access to a randomized response $Y$ derived from $X$. In the LDP framework, a central aggregator receives each user's randomized (privatized) data to be used for inferential tasks. In that sense, LDP differs from global DP \citep{dwork2006differential} where the aggregator privatizes operations on the sensitive dataset after it collects the sensitive data without noise. Hence LDP can be said to provide a stricter form of privacy and is used in cases where the aggregator may not be trustable \citep{Kasiviswanathan_et_al_2011}. 
Below, we give a more formal definition of $\epsilon$-LDP as a property that concerns a randomized mechanism.
\begin{defn}[Local differential privacy] \label{defn: LDP}
A randomized mechanism $\mathcal{M}: \mathcal{X} \mapsto \mathcal{Y}$ satisfies $\epsilon$-LDP if the following inequality holds for any pairs of inputs $x, x' \in \mathcal{X}$, and for any output (response) $y \in \mathcal{Y}$: 
\[
e^{-\epsilon} \leq \frac{\mathbb{P}(\mathcal{M}(x) = y)}{\mathbb{P}(\mathcal{M}(x') = y)} \leq e^{\epsilon} .
\]
\end{defn}
The definition of LDP is almost the same as that of global DP. The main difference is that, in the global DP, inputs $x, x'$ are two datasets that differ in only one individual's record, whereas in LDP, $x, x'$ are two different data points from $\mathcal{X}$. 

In Definition \ref{defn: LDP}, $\epsilon \geq 0$ is the privacy parameter. A smaller $\epsilon$ value provides stronger privacy. One main challenge in most differential privacy settings is to decide on the randomized mechanism. In the case of LDP, this is how an individual data point $X$ should be randomized. For a given randomized algorithm, too little randomization may not guarantee the privacy of individuals, whereas too severe randomization deteriorates the utility of the output of the randomized algorithm. Balancing these conflicting objectives (privacy vs utility) is the main goal of the research on estimation under privacy constraints.

In many cases, individuals' data points are collected sequentially. A basic example is opinion polling, where data is collected typically in time intervals of lengths in the order of hours or days. Personal data entered during registration is another example. For example, a hospital can collect patients' categorical data as they visit the hospital for the first time.

While sequential collection of individual data may make the estimation task under the LDP constraint harder, it may also offer an opportunity to adapt the randomized mechanism in time to improve the estimation quality. Motivated by that, in this paper, we address the problem of online Bayesian estimation of a categorical distribution ($\theta$) under $\epsilon$-LDP, while at the same time choosing the randomization mechanism adaptively so that the utility is improved continually in time.

\paragraph{Contribution:} 

\begin{figure}
    \centerline{\includegraphics[width=0.9\textwidth]{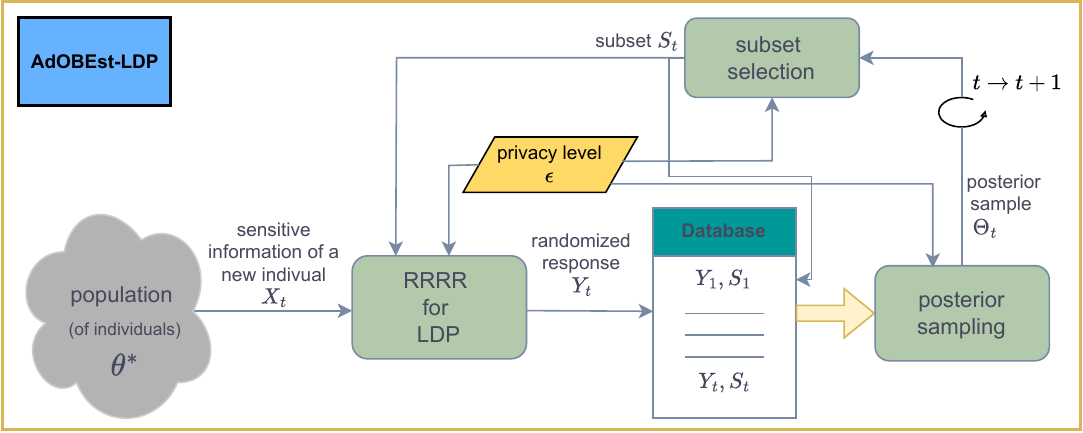}}
    \caption{AdOBEst-LDP: A framework for Adaptive and Online Bayesian Estimation of categorical distributions with Local Differential Privacy.}
        \label{fig: AdOBEst-LDP diagram}
\end{figure}

This paper presents AdOBEst-LDP: a new methodological framework for Adaptive Online Bayesian Frequency Estimation with LDP. A flowchart diagram of AdOBEst-LDP is given in Figure \ref{fig: AdOBEst-LDP diagram} to expose the reader to the main idea of the framework. The main idea of AdOBEst-LDP is to collect future privatized categorical data with high estimation utility based on the knowledge extracted from the previously collected privatized categorical data. To achieve this goal, AdOBEst-LDP continually adapts its randomized response mechanism to the estimation of $\theta$.

The development of AdOBEst-LDP offers three main contributions to the LDP literature.

\begin{itemize}
\item \textbf{A new randomized response mechanism:} AdOBEst-LDP uses a new adaptive \emph{Randomly Restricted Randomized Response} mechanism (RRRR) to produce randomized responses under $\epsilon$-LDP. RRRR is a generalization of the standard randomized response mechanism in that it restricts the response to a \emph{subset} of categories. This subset is selected such that the sensitive information $X$ of the next individual is likely contained in that subset. To ensure this, the subset selection step uses two inputs: (i) a sample for $\theta$ drawn from the posterior distribution of $\theta$ conditional on the past data, (ii) a utility function that scores the informativeness of the randomized response obtained from RRRR when it is run with a given subset. To that end, we propose several utility functions to score the informativeness of the randomized response. The utility functions are based on well-known tools and metrics from probability and statistics, such as Fisher information \citep{alparslan2022statistic, lopuhaa2022fisher, steinberger2024efficiency, yildirim2024differentially}, entropy, total variation distance, expected squared error, and probability of honest response, i.e., $Y = X$. We provide some insight into those utility functions both theoretically and numerically. Moreover, we also provide a computational complexity analysis for the proposed utility functions.

\item \textbf{Posterior sampling:} We equip AdOBEst-LDP with a scalable posterior sampling method for parameter estimation.  Bayesian estimation is a natural choice for inference when the data is corrupted or censored \citep{Kim_et_al_2011,  Liu_et_al_2022, Lone_et_al_2024} and such modification can be statistically modeled. In differential privacy settings, too, Bayesian inference is widely employed \citep{Williams_and_Mcsherry_2010, Karwa_et_al_2014, Foulds_et_al_2016, alparslan2022statistic} when the input data is shared with privacy-preserving noise. Standard MCMC methods, such as Gibbs sampling, have a computation complexity quadratic in the number of individuals whose data have been collected. As a remedy to this, similar to \citet{Mazumdar_et_al_2020}, we propose a \textit{stochastic gradient Langevin dynamics} (SGLD)-based algorithm to obtain approximate posterior samples \citep{Welling_and_Teh_2011}. By working on subsets of data, SGLD scales in time. 

The numerical experiments show that AdOBEst-LDP outperforms its non-adaptive counterpart when run with SGLD for posterior sampling. The results also suggest that the utility functions considered in this paper are robust and perform well. The MATLAB code at \url{https://github.com/soneraydin/AdOBEst_LDP} can be used to reproduce the results obtained in this paper.

\item \textbf{Convergence results:} Finally, we provide a theoretical analysis of AdOBEst-LDP. We prove two main results:
\begin{enumerate}[label=(\roman*)]
\item The targeted posterior distribution conditional on the generated observations by the adaptive scheme converges to the true parameter in probability in the number of observations, $n$. This convergence result owes mainly to the smoothness and a special form of concavity of the marginal log-likelihood function of the randomized responses. Another key factor is that the second moment of the sum up to time $n$ of the gradient of this log-marginal likelihood is $\mathcal{O}(n)$. 
\item If posterior sampling is performed exactly, the expected frequency of the algorithm choosing the best subset (according to the utility function) tends to $1$ as $n$ goes to $\infty$. 
\end{enumerate}
The theoretical results require fairly weak, realistic, and verifiable assumptions.
\end{itemize}

\paragraph{Outline:} In Section \ref{sec: Related Literature}, we discuss the earlier work related to ours. Section \ref{sec: Problem definition and general framework} presents LDP and the frequency estimation problem and introduces AdOBEst-LDP as a general framework. In Section \ref{sec: Constructing informative randomized response mechanisms}, we delve deeper into the details of AdOBEst-LDP by first presenting RRRR, the proposed randomized response mechanism, then explaining how it chooses an `optimal' subset of categories adaptively at each iteration. Section \ref{sec: Constructing informative randomized response mechanisms} also presents the utility metrics considered for choosing these subsets in this paper. In Section \ref{sec: Posterior sampling}, we provide the details of the posterior sampling methods considered in this paper, particularly SGLD. The theoretical analysis of AdOBEst-LDP is provided in Section \ref{sec: Theoretical analysis}. Section \ref{sec: Numerical Results} contains the numerical experiments. Finally, Section \ref{sec: Conclusion} provides some concluding remarks. All the proofs of the theoretical results are given in the appendices.

\section{Related Literature} \label{sec: Related Literature}
Frequency estimation under the LDP setting has been an increasingly popular research area in recent years. Along with its basic application (estimation of discrete probabilities from locally privatized data), it is also used for a wide range of other estimation and learning purposes such as estimation of confidence intervals and confidence sets for a population mean \citep{waudby2023nonparametric}, estimation or identification of heavy hitters \citep{zhu2023heavy} \citep{wang2023locally} \citep{jia2019calibrate}, estimation of quantiles \citep{cormode2022sample}, frequent itemset mining \citep{zhao2023hadamard}, estimation of degree distribution in social networks \citep{wang2023accurately}, distributed training of graph neural networks with categorical features and labels \citep{bhaila2024local}. The methods that are proposed for $\epsilon$-LDP frequency estimation also form the basis of more complex inferential tasks (with some modifications on these methods), such as the release of `marginals' (contingency tables) between multiple categorical features and their correlations, as in the work of \citet{cormode2018marginal}.

AdOBEst-LDP employs RRRR as its randomized mechanism to produce randomized responses. RRRR is a modified version of the \textit{Standard Randomized Response} mechanism (SRR) (also known as \textit{generalized randomized response}, \textit{$k$-randomized response}, and \textit{direct encoding} in the literature.) Given $X$ as its input, SRR outputs $X$ with probability $\frac{e^{\epsilon}}{e^{\epsilon} + K - 1}$, otherwise outputs one of the other categories at random. This is a well-studied mechanism in the DP literature, and the statistical properties of its basic version (such as its estimation variance) can be found in the works by \citep{wang2017locally} and \citep{wang2020locally}. When $K$ is large, the utility of SRR can be too low. RRRR in AdOBEst-LDP is designed to circumvent this problem by constraining its output to a subset of categories. Unlike SRR, the perturbation probability of responses in our algorithm changes adaptively, depending on the cardinality of the selected subset of categories (which we explain in detail in Section \ref{sec: Constructing informative randomized response mechanisms}) for the privatization of $X$, and the cardinality of its complementary set. 

The use of information metrics as utility functions in LDP protocols has been an active line of research in recent years. In the work of \citet{kairouz2016extremal}, information metrics like $f$-divergence and mutual information are used for selecting optimal LDP protocols. In the same vein, \citet{steinberger2024efficiency} uses Fisher Information as the utility metric for finding a nearly optimal LDP protocol for the frequency estimation problem, and \citet{lopuhaa2022fisher} uses it for comparing the utility of various LDP protocols for frequency estimation and finding the optimal one. In these works, the mentioned information metrics are used statically, i.e., to choose a protocol once and for all, for a given estimation task. The approaches in these works suffer from computational complexity for large values of $K$ because the search space for optimal protocols there grows in the order of $2^{K}$. In some other works, such as \citet{wang2016mutual}, a randomly sampled subset of size $k \leq K$ is used to improve the efficiency of this task, where the optimal $k$ is determined by maximizing the \textit{mutual information} between real data and the privatized data. However, this approach is also static as the optimal subset size $k$ is selected only once, and the optimization procedure only determines $k$ and not the subset itself. Unlike those static approaches, AdOBEst-LDP dynamically uses the information metric (such as the Fisher Information matrix and the other alternatives in Section \ref{sec: Subset selection for RRRR}) to select the optimal subset at each time step. In addition, in the subset selection step of AdOBEst-LDP, only $K$ candidate subsets are compared in terms of their utilities at each iteration, enabling computational tractability. This way of tackling the problem requires computing the given information metric for only $K$ times at each iteration. We will provide further details of this approach in Section \ref{sec: Subset selection for RRRR} and provide a computational complexity analysis in Section \ref{sec: Computational complexity of utility functions}. 

Another use of the Fisher Information in the LDP literature is for bounding the estimation error for a given LDP protocol. For example, \citet{barnes2020fisher} uses Fisher Information inside van Trees inequality, the Bayesian version of the Cramér-Rao bound \citep{gill1995applications}, for bounding the estimation error of various LDP protocols for Gaussian mean estimation and frequency estimation. Again, their work provides rules for choosing optimal protocols for a given $\epsilon$ in a static way. As a similar example, \citet{acharya2023unified} derives a general \textit{information contraction bound} for parameter estimation problems under LDP and shows its relation to van Trees inequality as its special case. To our knowledge, our approach is the first one that adaptively uses a utility metric to dynamically update the inner workings of an LDP protocol for estimating categorical distributions.

The idea of building adaptive mechanisms for improved estimation under the LDP has been studied in the literature, although the focus and methodology of those works differ from ours. For example, \citet{joseph2019locally} proposed a two-step adaptive method to estimate the unknown mean parameter of data from Gaussian distribution. In this method, the users are split into two groups, an initial mean estimate is obtained from the perturbed data of the first group, and the data from the second group is transformed adaptively according to that initial estimate. Similarly, \citet{wei2024aaa} proposed another two-step adaptive method for the mean estimation problem, in which the aggregator first computes a rough distribution estimate from the noisy data of a small sample of users, which is then used for adjusting the amount of perturbation for the data of remaining users. While \citet{joseph2019locally, wei2024aaa} consider a two-stage method, AdOBEst-LDP seeks to adapt continually by updating its LDP mechanism each time an individual's information is collected. Similar to our work, \citet{yildirim2024differentially} has recently proposed an adaptive LDP mechanism for online parameter estimation for continuous distributions. The LDP mechanism of \citet{yildirim2024differentially} contains a truncation step with boundaries adapted to the estimate from the past data according to a utility function based on the Fisher information. Unfortunately, the parameter estimation step of \citet{yildirim2024differentially} does not scale in time. Differently from \citet{yildirim2024differentially}, AdOBEst-LDP focuses on categorical distributions, considers several other utility functions to update its LDP mechanism, employs a scalable parameter estimation step, and its performance is backed up with theoretical results.

\section{Problem definition and general framework} \label{sec: Problem definition and general framework}
Suppose we are interested in a discrete probability distribution $\mathcal{P}$ of a certain form of sensitive categorical information $X \in [K] := \{1, \ldots, K\}$ of individuals in a population. Hence, $\mathcal{P} $ is a categorical distribution $\text{Cat}(\theta^{\ast})$ with a probability vector 
\[
\theta^{\ast} := (\theta^{\ast}_{1}, \ldots, \theta^{\ast}_{K}) \in \Delta,
\]
where $\Delta$ is the $(K-1)$-dimensional probability simplex,
\[
\Delta := \left\{ \theta \in \mathbb{R}^{K}: \sum_{k = 1}^{K} \theta_{k} = 1 \text{ and } \theta_{k} \geq 0 \text{ for } k \in [K] \right\}.
\]
We assume a setting where individuals' sensitive data are collected \emph{privately} and \emph{sequentially in time}. The privatization is performed via a randomized algorithm that, upon taking a category index in $[K]$ as an input, returns a random category index in $[K]$ such that the whole data collection process is $\epsilon$-LDP. (See Definition \ref{defn: LDP}.) Let $X_{t}$ and $Y_{t}$ be the private information and randomized responses of individual $t$, respectively. According to Definition \ref{defn: LDP} for LDP, the following inequality must be satisfied for all triples $(x, x', y) \in [K]^{3}$ for the randomized mechanism to be $\epsilon$-LDP.
\begin{equation} \label{eq: LDP of the randomized mechanism}
\mathbb{P}(Y_{t} = y | X_{t} = x) \leq e^{\epsilon} \mathbb{P}(Y_{t} = y | X_{t} = x').
\end{equation}
The inferential goal is to estimate $\theta^{\ast}$ sequentially based on the responses $Y_{1}, Y_{2}, \ldots$, and the mechanisms $\mathcal{M}_{1}, \mathcal{M}_{2}, \ldots$ that are used to generate those responses. Specifically, Bayesian estimation is considered, whereby the target is the posterior distribution, denoted by $\Pi(\mathrm{d} \theta | Y_{1:n}, \mathcal{M}_{1:n})$, given a prior probability distribution with pdf $\eta(\theta)$ on $\Delta$.

This paper concerns the Bayesian estimation of $\theta$ while adapting the randomized mechanism to improve the estimation utility continually. We propose a general framework called AdOBEst-LDP, in which the randomized mechanism at time $t$ is adapted to the data collected until time $t-1$. AdOBEst-LDP is outlined in Algorithm \ref{alg: Differentially private online learning}.

\begin{algorithm}[t]
\caption{AdOBEst-LDP: Adaptive Online Bayesian Estimation with LDP}
\label{alg: Differentially private online learning}
Initialization: Start with an initial estimator $\Theta_{0} = \theta_{\text{init}}$.

\For{$t = 1, 2, \ldots$}{
\textbf{Step 1: Adapting the LDP mechanism:} Based on $\Theta_{t-1}$, determine the $\epsilon$-LDP mechanism $\mathcal{M}_{t}$ for the next individual according to a utility metric.

\textbf{Step 2: LDP response generation} The sensitive information $X_{t}$ of individual $t$ is shared as $Y_{t}$ using the $\epsilon$-LDP mechanism $\mathcal{M}_{t}$.

\textbf{Step 3:} Draw a sample (approximately) from the posterior distribution
\begin{equation*} 
\Theta_{t} \sim \Pi(\cdot | Y_{1:t}, \mathcal{M}_{1:t}).
\end{equation*}
}
\end{algorithm}

Algorithm \ref{alg: Differentially private online learning} is fairly general, and it does not describe how to choose the $\epsilon$-LDP mechanism $\mathcal{M}_{t}$ at time $t$, nor does it provide the details of the posterior sampling. However, it is still worth making some critical observations about the nature of the algorithm. Firstly, at time $t$ the selection of the $\epsilon$-LDP mechanism in Step 1 relies on the posterior sample $\Theta_{t-1}$, which serves as an \emph{estimator} of the true parameter $\theta^{\ast}$ based on the past observations. As we shall see in Section \ref{sec: Constructing informative randomized response mechanisms}, at Step 1 the `best' $\epsilon$-LDP mechanism is chosen from a set of candidate LDP mechanisms according to a utility function. This step is relevant only when $\Theta_{t-1}$ is a reliable estimator of $\theta^{\ast}$. In other words, Step 1 `exploits' the estimator $\Theta_{t-1}$. Moreover, the random nature of posterior sampling prevents having too much confidence in the current estimator $\Theta_{t-1}$ and enables a certain degree of `exploration.' In conclusion, Algorithm \ref{alg: Differentially private online learning} utilizes an `exploration-exploitation' approach reminiscent of reinforcement learning. In particular, posterior sampling in Step 3 suggests a strong parallelism between AdOBEst-LDP and the well-known exploration-exploitation approach called Thompson sampling \citep{Russo_et_al_2018}.

The details of Steps 1-2 and Step 3 of Algorithm \ref{alg: Differentially private online learning} are given in Sections \ref{sec: Constructing informative randomized response mechanisms} and \ref{sec: Posterior sampling}, respectively. 

\section{Constructing informative randomized response mechanisms} \label{sec: Constructing informative randomized response mechanisms}
In this section, we describe Steps 1-2 of AdOBEst-LDP in Algorithm \ref{alg: Differentially private online learning} where the $\epsilon$-LDP mechanism $\mathcal{M}_{t}$ is selected at time $t$ based on the posterior sample $\Theta_{t}$ and a randomized response is generated using $\mathcal{M}_{t}$. For ease of exposition, we will drop the time index $t$ throughout the section and let $\Theta_{t-1} = \theta$.

Recall from Definition \ref{defn: LDP} that an $\epsilon$-LDP randomized mechanism is associated with a conditional probability distribution that satisfies \eqref{eq: LDP of the randomized mechanism}. An $\epsilon$-LDP mechanism is not unique. One such mechanism is the standard randomized response mechanism (SRR). For subsequent use, it is convenient to define SRR generally: We let $\texttt{SRR}(X; \Omega, \epsilon)$ the output of SRR which operates on the set $\Omega$ with LDP parameter $\epsilon$ when the input is $X \in \Omega$. Then, we have
\begin{equation} \label{eq: SRR output}
Y = \text{SRR}(X; \Omega, \epsilon) = \begin{cases} X & \text{w.p. } e^{\epsilon}/(e^{\epsilon}+|\Omega|-1) \\
\sim \text{Uniform}(\Omega/\{X\}) & \text{ else }\end{cases}.
\end{equation}

We aim to develop an alternative randomized mechanism whose response $Y$ is more informative about $\theta^{\ast}$ than the one generated as $Y = \text{SRR}(X;  [K], \epsilon)$. The main idea is as follows. Supposing that the posterior sample $\Theta_{t-1} = \theta$ is an accurate estimate of $\theta^{\ast}$, it is reasonable to aim for the `best' $\epsilon$-LDP mechanism (among a range of candidates) which would maximize the (estimation) utility of $Y$ if the true parameter were $\theta^{\ast} = \theta$. We follow this main idea to develop the proposed $\epsilon$-LDP mechanism.

\subsection{The randomly restricted randomized response (RRRR) mechanism}
Given $\Theta_{t-1} = \theta \in \Delta$, an informative randomized response mechanism can be constructed by considering a \emph{high-probability set} $S \subset [K]$ and a \emph{low-probability set} $S^{c} = [K]/S$ for $X$ (according to $\theta$). Then, a sensible alternative to $\text{SRR}(X; [K], \epsilon)$ would be to confine the randomized response to the set $S$ (unioned by a random element from $S^{c}$ to remain LDP). The expected benefit of this approach is due to (i) using less amount of randomization since $|S| < K$, and thus (ii) having an informative response when $X \in S$, which happens with a high probability. Based on this approach, we propose RRRR, whose precise steps are given in Algorithm \ref{alg: RRRR}.

\begin{algorithm}
\caption{RRRR}
\label{alg: RRRR}
\KwIn{Sample space size $K$, a subset $S \subset [K]$, privacy parameters $\epsilon_{1}, \epsilon_{2} > 0$, input $X \in [K]$}
\KwOut{Randomized response $Y \in [K]$}
\uIf{$X \in S$}{
Draw $R \sim \textup{Uniform}(S^{c})$.

Set $Y = \texttt{SRR}(X; S \cup \{R\}, \epsilon_{1})$ as in \eqref{eq: SRR output}.
}
\Else{
Set $R = \texttt{SRR}(X; S^{c}, \epsilon_{2})$ as in \eqref{eq: SRR output}.

Set $Y = \texttt{SRR}(R; S \cup \{R\}, \epsilon_{1})$ as in \eqref{eq: SRR output}.
}
\Return{$Y$}

\end{algorithm}

RRRR has three algorithmic parameters: a subset $S$ of $[K]$ and two privacy parameters $\epsilon_{1}$ and $\epsilon_{2}$ which operates on $S$ and $S^{c}$, respectively. Theorem \ref{thm: RRRR LDP} states the necessary conditions for $\epsilon_{1}$  and $\epsilon_{2}$  for RRRR to be $\epsilon$-LDP. A proof of Theorem \ref{thm: RRRR LDP} is given in Appendix \ref{appndx: Proofs for LDP of RRRR}.

\begin{thm} \label{thm: RRRR LDP}
\textup{RRRR} is $\epsilon$-DP if $\epsilon_{1} \leq \epsilon$ and 
\begin{equation} \label{eq: epsilon2 after epsilon1}
\epsilon_{2} = \begin{cases} \min\left\{ \epsilon, \ln \frac{|S^{c}|-1}{e^{\epsilon_{1} - \epsilon} |S^{c}| -1 }  \right\}& \text{for } \epsilon - \epsilon_{1} < \ln |S^{c}| \text{ and } |S| > 0\\
\epsilon & \text{else}
\end{cases}.
\end{equation}
\end{thm}
Note that when $S = \emptyset$ and $\epsilon_{2} = \epsilon$, RRRR reduces to SRR.

\subsection{Choosing the privacy parameters $\epsilon_{1}$, $\epsilon_{2}$}
We elaborate on the choice of $\epsilon_{1}$ and $\epsilon_{2}$ in the light of Theorem \ref{thm: RRRR LDP}. In RRRR, the probability of an honest response, i.e., $X = Y$, given $X \in S$, is
\[
\mathbb{P}(Y = X | X \in S) = \frac{e^{\epsilon_{1}}}{e^{\epsilon_{1}} + |S|},
\]
which should be contrasted to $e^{\epsilon}/(e^{\epsilon}+K-1)$, which would be the probability if $Y = \text{SRR}(X; [K], \epsilon)$. Anticipating that $\{ X \in S\}$ is likely, one should at least aim for $\epsilon_{1}$ that satisfies $\mathbb{P}(X = Y | X \in S) \geq e^{\epsilon}/(e^{\epsilon}+K-1)$ for RRRR to be relevant. This is equivalent to
\begin{equation} \label{eq: condition for accuracy}
\epsilon_{1} \geq \epsilon + \ln |S| -  \ln (K-1).
\end{equation}
Taking into account also the constraint that $\epsilon_{1} \leq \epsilon$ (by Theorem \ref{thm: RRRR LDP}), we suggest $\epsilon_{1} = \kappa \epsilon$, where $\kappa \in (0,1)$ is a number close to $1$, such as $0.9$, to ensure \eqref{eq: condition for accuracy} with a significant margin. (It is possible to choose $\kappa = 1$; however, again by Theorem \ref{thm: RRRR LDP}, this requires that $\epsilon_{2} = 0$, which renders $Y$ completely uninformative when $X \notin S$.) In Section \ref{sec: Numerical Results}, we discuss the choice of $\kappa$ in more detail.

For the next section, we assume a fixed $\kappa \in (0, 1)$, and set $\epsilon_{1} = \kappa \epsilon$; and we focus on the selection of $S$. 

\subsection{Subset selection for RRRR} \label{sec: Subset selection for RRRR}

Let $\texttt{RRRR}(X; S, \epsilon)$ be the random output of RRRR that achieves $\epsilon$-LDP by using the subset $S$ and the privacy parameters $\epsilon_{1} = \kappa \epsilon$ and $\epsilon_{2}$ as in \eqref{eq: epsilon2 after epsilon1} when the input is $X$. Furthermore, let $U(\theta, S, \epsilon)$ be the (inferential) `utility' of $Y = \texttt{RRRR}(X; S, \epsilon)$ when $X \sim \text{Cat}(\theta)$. One would like to choose $S$ that maximizes $U(\theta, S, \epsilon)$. (One could also seek to optimize $\kappa$ in $\epsilon_{1} = \kappa \epsilon$, too, however with the expense of additional computation.)

However, since there are $2^{K} - 1$ feasible choices for $S$, one must confine the search space for $S$ in practice. As discussed above, RRRR becomes most relevant when the set $S$ is a high-probability set. 
Therefore, for a given $\theta$, we confine the choices for $S$ to
\begin{equation} \label{eq: Ss considered for RRRR}
S_{\theta, k} := \{\sigma_{\theta}(1), \sigma_{\theta}(2), \ldots, \sigma_{\theta}(k)\}, \quad k = 1, \ldots, K.
\end{equation}
where $\sigma_{\theta} := (\sigma_{\theta}(1), \ldots, \sigma_{\theta}(K))$ be the permutation vector for $\theta$ so that $\theta_{\sigma_{\theta}(1)} \geq \ldots \geq \theta_{\sigma_{\theta}(K)}$. 

Then the subset selection problem can be formulated as finding
\begin{equation}\label{eq: maximization of utility to choose S}
k^{\ast} = \arg \max_{k \in \{0, \ldots, K-1\}} U(\theta, S_{k, \theta}, \epsilon).
\end{equation}
The alternatives in \eqref{eq: Ss considered for RRRR} can be justified. Since $S_{\theta, k}$ contains the indices of the $k$ highest-valued components of $\theta^{\ast}$, it is expected to cover a large portion of the total probability for $X$. This can be the case even for a small value of $k$ relative to $K$ when the components of $\theta^{\ast}$ are not evenly distributed. Also, the alternatives cover the basic SRR, which is obtained with $k = 0$ (leading to $S = \emptyset$ and $\epsilon_{2} = \epsilon$).

In the subsequent sections, we present six different utility functions $U(\theta, S, \epsilon)$ and justify their relevance to estimation; the usefulness of the proposed functions is also demonstrated in the numerical experiments.  

\subsubsection{Fisher information matrix} \label{sec: Fisher Information Matrix}
The first utility function under consideration is based on the Fisher information matrix at $\theta$ according to the distribution of $Y$ given $\theta$. It is well-known that the inverse of the Fisher information matrix sets the Cramer-Rao lower bound for the variance of an unbiased estimator. Hence, the Fisher information can be regarded as a reasonable metric to quantify the information contained in $Y$ about $\theta$. This approach is adopted in \citet{lopuhaa2022fisher, steinberger2024efficiency} for LDP applications for estimating discrete distributions, and \citet{alparslan2022statistic, yildirim2024differentially} in similar problems involving parametric continuous distributions. 

For a given $\theta \in \Delta$, let $F(\theta; S, \epsilon)$ be the Fisher information matrix evaluated at $\theta$ when $X \sim \text{Cat}(\theta)$ and $Y = \texttt{RRRR}(X; S, \epsilon)$. Let 
\[
g_{S, \epsilon}(y | x) := \mathbb{P}(Y = y | X = x)
\]
when $Y = \texttt{RRRR}(X; S, \epsilon)$. The following result states $F(\theta; S, \epsilon)$ in terms of $g_{S, \epsilon}$ and $\theta$. The result is derived in \citet{lopuhaa2022fisher}; we also give a simple proof in Appendix \ref{appndx: Proofs about utility functions}. Note that $F(\theta; S, \epsilon)$ is $(K-1)\times(K-1)$ since $\theta$ has $K-1$ free components and $\theta_{K} = 1 - \sum_{i = 1}^{K-1} \theta_{i}$. 
\begin{prop} \label{prop: FIM of RRRR}
The Fisher information matrix for RRRR is given by
\begin{equation} \label{eq: FIM of RRRR}
F(\theta; S, \epsilon) = A_{S, \epsilon}^{\top} D_{\theta}^{-1} A_{S, \epsilon}, 
\end{equation}
where $A_{S, \epsilon}$ is a $K \times (K -1)$ matrix whose entries are $A_{S, \epsilon}(i, j) := g_{S, \epsilon}(i | j) - g_{S, \epsilon}(i | K)$ and $D_{\theta}$ is a $K \times K$ diagonal matrix with elements $D_{\theta}(i, i) := \sum_{j = 1}^{K} g_{S, \epsilon}(i | j) \theta_{j}$.
\end{prop}

We define the following utility function based on the Fisher information
\begin{equation} \label{eq: utility FIM}
U_{1}(\theta, S, \epsilon) := -\textup{Tr} \left[ F^{-1}(\theta; S, \epsilon)\right].
\end{equation}
This utility function depends on the Fisher information differently from \citet{lopuhaa2022fisher, steinberger2024efficiency}, who considered the determinant of the FIM as the utility function. The rationale behind \eqref{eq: utility FIM} is that the for an unbiased estimator $\hat{\theta}(Y)$ of $\theta^{\ast}$ based on $Y = \texttt{RRRR}(X; S, \epsilon)$, the expected mean squared error is bounded by $E_{\theta^{\ast}} [ \| \hat{\theta}(Y) - \theta^{\ast} \|^{2} ] \leq \textup{Tr} \left[ F^{-1}(\theta^{\ast}; S, \epsilon)\right]$. For the utility function in \eqref{eq: utility FIM} to be well-defined, the FIM needs to be invertible. Proposition \ref{prop: FIM is invertible}, proven in Appendix \ref{appndx: Proofs about utility functions}, states that this is indeed the case. 

\begin{prop} \label{prop: FIM is invertible}
$F(\theta; S, \epsilon)$ in \eqref{eq: FIM of RRRR} is invertible for all $\theta \in \Delta$, $S \subset [K]$, and $\epsilon_{1}, \epsilon_{2} > 0$.
\end{prop}

\subsubsection{Entropy of randomized response} \label{sec: Entropy of randomized response}
For discrete distributions, entropy measures \textit{uniformity}. Hence, in the LDP framework, a lower entropy for the randomized response $Y$ implies a more informative $Y$. Based on that observation, a utility function can be defined as the negative entropy of the marginal distribution of $Y$,
\begin{equation*}
U_{2}(\theta, S, \epsilon) := \sum_{y= 1}^{K} \ln h_{S, \epsilon}(y | \theta) h_{S, \epsilon}(y | \theta),
\end{equation*}
where $h_{S, \epsilon}(y | \theta)$ is the marginal probability of $Y = y$ given $\theta$,
\[
h_{S, \epsilon}(y | \theta) := \sum_{x = 1}^{K}  g_{S, \epsilon}(y | x) \theta_{x}.
\]

\subsubsection{Total variation distance} \label{sec: Total variation distance}
The TV distance between two discrete probability distributions $\mu, \nu$ on $[K]$ is given by
\[
\textup{TV}(\mu, \nu) := \frac{1}{2} \sum_{k =1}^{K} | \mu(x) - \nu(x) |.
\]
We consider two utility functions based on TV distance. The first function arises from the observation that a more informative response $Y$ generally leads to a larger change in the posterior distribution of $X$ given $Y, \theta$, 
\begin{equation}\label{eq: post of x given y and theta}
p_{S, \epsilon}(x | y, \theta) := \frac{\theta_{x} \cdot g_{S, \epsilon}(y | x)}{h_{S, \epsilon}(y | \theta)}, \quad x = 1, \ldots, K,
\end{equation}
relative to its prior $\text{Cat}(\theta)$. The expected amount of change can be formulated as the expectation of the TV distance between the prior and posterior distributions with respect to the marginal distribution of $Y$ given $\theta$. Then, a utility function can be defined as 
\begin{align*}
U_{3}(\theta, S, \epsilon) &:= \mathbb{E}_{\theta} \left[ \text{TV}(p_{S, \epsilon}(\cdot | Y, \theta), \text{Cat}(\theta)) \right] \\
&= \frac{1}{2} \sum_{x = 1}^{K}  \sum_{y = 1}^{K}\left\vert g_{S, \epsilon}(y | x) \theta_{x} - h_{S, \epsilon}(y | \theta) \theta_{x}  \right\vert.
\end{align*}

Another utility function is related to the TV distance between the marginal probability distributions of $X$ given $\theta$ and $Y$ given $\theta$. Since $X$ is more informative about $\theta$ than the randomized response $Y$, the mentioned TV distance is desired to be as small as possible. Hence, a utility function may be formulated as
\begin{align*}
U_{4}(\theta, S, \epsilon) &:= -\text{TV}(h_{S, \epsilon}(\cdot | \theta), \text{Cat}(\theta)) \label{eq: utility TV-2}\\
&= -\frac{1}{2}\sum_{i = 1}^{K} | h_{S, \epsilon}(i | \theta) - \theta_{i}|. \nonumber
\end{align*}

\subsubsection{Expected mean squared error} \label{sec: Expected mean squared error for X}
One can also wish to choose $S$ such that the Bayesian estimator of $X$ given $Y$ has the lowest expected squared error. Specifically, given $k \in [K]$ let $e_{k}$ be a $K \times 1$ vector of $0$s except that $e_{k} = 1$. A utility function can be defined based on that as
\begin{equation} \label{eq: utility MSE}
U_{5}(\theta, S, \epsilon) := - \arg \min_{\widehat{e_{X}}} \mathbb{E}_{\theta} \left[ \| e_{X} - \widehat{e_{X}}(Y) \|^{2} \right],
\end{equation}
where $\mathbb{E}_{\theta} \left[ \| e_{X} - \widehat{e_{X}}(Y) \|^{2} \right]$ is the mean squared error for the estimator $\widehat{e_{X}}$ of $e_{X}$ given $Y$ when $X \sim \text{Cat}(\theta)$ and $Y = \texttt{RRRR}(X; S, \epsilon)$, which is known to be minimized when $\widehat{e_{X}}$ is the Bayesian estimator of $e_{X}$. Proposition \ref{prop: utility MSE} provides an explicit formula for this utility function. A proof is given in Appendix \ref{appndx: Proofs about utility functions}.
\begin{prop} \label{prop: utility MSE}
For the utility function in \eqref{eq: utility MSE}, we have
\[
U_{5}(\theta, S, \epsilon) = \sum_{y = 1}^{K} \sum_{x = 1}^{K} \frac{g_{S, \epsilon}(y| x)^{2} \theta_{x}^{2}}{h_{S, \epsilon}(y | \theta)} - 1.
\]
\end{prop}

\subsubsection{Probability of honest response} \label{sec: Probability of honest response}
Our last alternative for the utility function is a simple yet intuitive one, which is the probability of an honest response, i.e., 
\begin{equation} \label{eq: utility x y same}
U_{6}(\theta, S, \epsilon) := \mathbb{P}_{\theta}(Y = X | S).
\end{equation}
This probability is explicitly given by
\begin{align*}
\mathbb{P}_{\theta}(Y = X | S) &= \mathbb{P}(Y = X | X\in S) \mathbb{P}_{\theta}(X \in S )+  \mathbb{P}(Y = X | X \notin S) \mathbb{P}_{\theta}(X \notin S ) \\
&= \frac{e^{\epsilon_{1}}}{e^{\epsilon_{1}} + |S|} \left( \sum_{i\in S} \theta_{i} + \frac{e^{\epsilon_{2}}}{e^{\epsilon_{2}} + K-|S|-1}  \sum_{i \notin S} \theta_{i} \right).
\end{align*}
Recall that, for computational tractability, we confined the possible sets for $S$ to the subsets $\{\sigma_{\theta}(1), \ldots, \sigma_{\theta}(k)\}$, $k= 0, \ldots, K-1$ and select $S$ by solving the maximization problem in \eqref{eq: maximization of utility to choose S}. Remarkably, if $U_{6}(\theta, S, \epsilon)$ is used for the utility function, the restricted maximization \eqref{eq: maximization of utility to choose S} is equivalent to \emph{global} maximization, i.e., finding the best $S$ among all the $2^{K}$ possible subsets $S$. We state this as a theorem and prove it in Appendix \ref{appndx: Proofs about utility functions}.
\begin{thm} \label{thm: PXY global opt}
For the utility function $U_{6}(\theta, S, \epsilon)$ in \eqref{eq: utility x y same} and $S_{k, \theta}$s in \eqref{eq: Ss considered for RRRR}, we have
\[
\max_{k = 0, \ldots, K-1} U_{6}(\theta, S_{k, \theta}, \epsilon) = \max_{S \subset [K]} U_{6}(\theta, S, \epsilon).
\]
\end{thm}

\subsubsection{Semi-adaptive approach} \label{sec: Semi-adaptive approach}
We also consider a semi-adaptive approach which uses a fixed parameter $\alpha \in (0, 1)$ to select the smallest $S_{k, \theta}$ in \eqref{eq: Ss considered for RRRR} such that $\mathbb{P}_{\theta}(X \in S_{k, \theta}) \geq \alpha$, that is, $S = \{\sigma_{\theta}(1), \ldots, \sigma_{\theta}(k^{\ast})\}$ is taken such that
\[
\mathbb{P}_{\theta}(X \in \{\sigma_{\theta}(1), \ldots, \sigma_{\theta}(k^{\ast}-1)\}) < \alpha \text{ and } \mathbb{P}_{\theta}(X \in \{\sigma_{\theta}(1), \ldots, \sigma_{\theta}(k^{\ast})\}) \geq \alpha.
\] 
Again, the idea is to randomize the most likely values of $X$ with a high accuracy. The approach forms the subset $S$ by including values for $X$ in descending order of their probabilities (given by $\theta$) until the cumulative probability exceeds $\alpha$. In that way, it is expected to have set $S$ that is small-sized (especially when $\theta$ is unbalanced) and captures the most likely values of $X$. The resulting $S$ has varying cardinality depending on the sampled $\theta$ at the current time step.

We call this approach ``semi-adaptive'' because, while it still adapts to $\theta$, it uses the fixed parameter $\alpha$. As we will see in Section \ref{sec: Numerical Results}, the best $\alpha$ depends on various parameters such as $\epsilon$, $K$, and the degree of evenness in $\theta$.

\subsection{Computational complexity of utility functions} \label{sec: Computational complexity of utility functions}
We now provide the computational complexity analysis of the utility metrics presented in Section \ref{sec: Fisher Information Matrix}-\ref{sec: Probability of honest response}, and that of the semi-adaptive approach in Section \ref{sec: Semi-adaptive approach}, as a function of $K$. The first row of Table \ref{tab: computational complexity} shows the computational complexities of calculating the utility function for a fixed $S$, and the second row shows the complexities of choosing the best $S$ according to \eqref{eq: maximization of utility to choose S}. To find \eqref{eq: maximization of utility to choose S}, the utility function generally needs to be calculated $K$ times, which explains the additional $K$ factor in the computational complexities in the second row.

The least demanding utility function is $U_{6}$, that is based on $\mathbb{P}_{\theta}(Y = X)$, whose complexity is $\mathcal{O}(K)$. Moreover, finding the best $S$ can also be done in $\mathcal{O}(K)$ time because one can compute this utility metric for all $k = 0, \ldots, K-1$ by starting with $S = \emptyset$ and expanding it incrementally. Also note that the semi-adaptive approach does not use a utility metric and finding $k^{\ast}$ can be done in $\mathcal{O}(K)$ time by summing the components of $\theta$ from largest to smallest until the cumulative sum exceeds the given $\alpha$ parameter. So, its complexity is  $\mathcal{O}(K)$.

For all these approaches, it is additionally required to sort $\theta$ beforehand, which is an $\mathcal{O}(K \ln K)$ operation with an efficient sorting algorithm like \textit{merge sort}.

In practice, one can choose among these utility functions depending on the nature of the application.  When the number of categories $K$ or the arrival rate of sensitive data is large, we suggest using $U_{6}$ or a semi-adaptive approach. When $K$ and the arrival rate of the personal data are both small, the more computationally demanding utility functions can also be used.

\begin{table}
    \centering
    \caption{Computational complexity of utility functions and choosing $S$}
    \begin{tabular}{c c c c c c c c}
       \toprule
          &  Fisher & Entropy & $\text{TV}_{1}$ & $\text{TV}_{2}$ & $\text{MSE}$ & $\mathbb{P}_{\theta}(Y = X)$  & Semi-adaptive \\ \midrule
       Computing utility & $\mathcal{O}(K^{3})$ & $\mathcal{O}(K^{2})$ & $\mathcal{O}(K^{2})$ & $\mathcal{O}(K^{2})$ & $\mathcal{O}(K^{2})$ & $\mathcal{O}(K)$ &  NA  \\ \hline
        Choosing $S$ & $\mathcal{O}(K^{4})$ & $\mathcal{O}(K^{3})$ & $\mathcal{O}(K^{3})$ & $\mathcal{O}(K^{3})$ & $\mathcal{O}(K^{3})$ & $\mathcal{O}(K)$ &  $\mathcal{O}(K)$  \\ \bottomrule
    \end{tabular}
    \label{tab: computational complexity}
\end{table}

\begin{ex}[\textbf{Numerical illustration}] \label{ex: numerical illustration}
We close this section with an example that shows the benefit of \texttt{RRRR} and the role of $S$. We consider $\theta$ values such that $\theta_{i}/\theta_{i+1}$ is constant for $i = 1, \ldots, K-1$. The ratio $\theta_{i}/\theta_{i+1}$ controls the degree of `evenness' in $\theta$: The smaller ratio indicates a more evenly distributed $\theta$. Note that $\theta$ is already ordered in this example; hence, we consider using $S = \{1, \ldots, k\}$ which has the $k$ most likely values for $X$ according to $\theta$. Also, for a given $\epsilon$, we fix $\epsilon_{1} = 0.9 \epsilon$ and set $\epsilon_{2}$ according to \eqref{eq: epsilon2 after epsilon1}. 

Figure \ref{fig: Prob X is Y} shows, for a fixed $\epsilon$ and $K = 20$, and various values of $k$, the probability of the randomized response being equal to the sensitive information, i.e., $\mathbb{P}_{\theta}(Y = X)$ vs $\theta_{i}/\theta_{i+1}$ when $S = \{1, \ldots, k\}$ in RRRR. (Recall that this probability corresponds to $U_{6}(\theta, S, \epsilon)$.) Comparing this probability with $e^{\epsilon}/(e^{\epsilon}+K-1)$, the probability obtained with $Y = \text{SRR}(X; [K], \epsilon)$, it can be observed that RRRR can do significantly better than SRR if $k$ can be chosen suitably. The plots demonstrate that the ``suitable'' $k$ depends on $\theta$: While the best $k$ tends to be larger for more even $\theta$, small $k$ becomes the better choice for non-even $\theta$ (large $\theta_{i}/\theta_{i+1}$). This is because, when $\theta_{i}/\theta_{i+1}$ is large, the probability is concentrated on just a few components, and $S$ with a small $k$ captures most of the probability. Moreover, the plots for $\epsilon = 1$ and $\epsilon = 5$ also show the effect of the level of privacy. In more challenging scenarios where $\epsilon$ is smaller, the gain obtained by RRRR compared to SRR is bigger.

\begin{figure}[h]
\centerline{\includegraphics[width=0.5\textwidth]{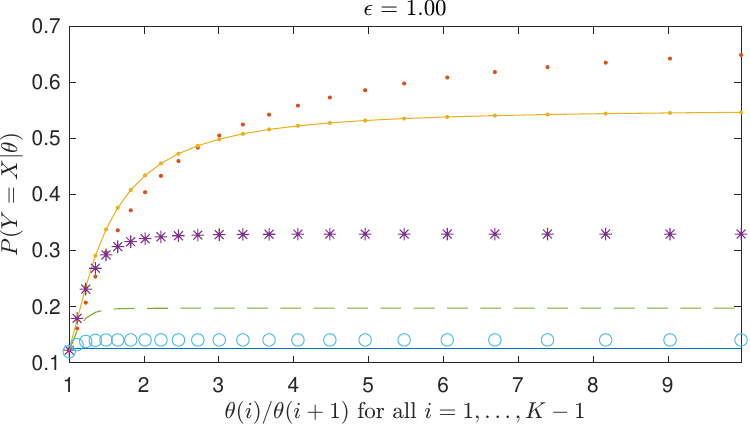} \includegraphics[width=0.5\textwidth]{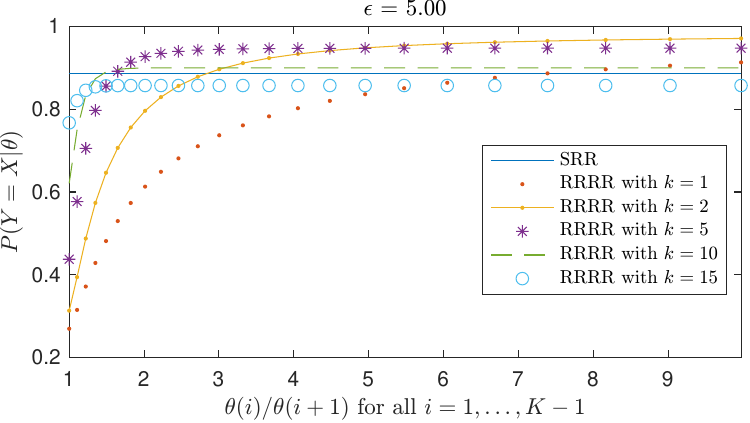}}
\caption{$\mathbb{P}_{\theta}(Y = X)$ vs $\theta_{i}/\theta_{i+1}$ for all $i = 1, \ldots, K-1$ with $K = 20$. Left: $\epsilon = 1$, Right: $\epsilon = 5$.}
\label{fig: Prob X is Y}
\end{figure}
\end{ex}

\section{Posterior sampling} \label{sec: Posterior sampling}
Steps 1-2 of AdOBEst-LDP in Algorithm \ref{alg: Differentially private online learning} were detailed in the previous section. In this section, we provide the details of Step 3.

Step 3 of AdOBEst-LDP requires sampling from the posterior distribution $\Pi(\cdot | Y_{1:n}, S_{1:n})$ of $\theta$ given $Y_{1:n}$ and $S_{1:n}$ for $n \geq 1$, where $S_{t}$ is the subset selected at time $t$ to generate $Y_{t}$ from $X_{t}$. Let $\pi(\theta | Y_{1:n}, S_{1:n})$ denote the pdf of $\Pi(\cdot | Y_{1:n}, S_{1:n})$. Given $Y_{1:n} = y_{1:n}$ and $S_{1:n} = s_{1:n}$, the posterior density can be written as 
\begin{align}
\pi(\theta | y_{1:n}, s_{1:n}) & \propto \eta(\theta) \prod_{t = 1}^{n} h_{s_{t}, \epsilon}(y_{t}| \theta). \label{eq: posterior of theta given Y and S}
\end{align}
Note that the right-hand side does not include a transition probability for $S_{t}$'s because the sampling procedure of $S_{t}$ given $Y_{1:t-1}$ and $S_{1:t-1}$ does \emph{not} depend on $\theta^{\ast}$. Furthermore, we assume that the prior distribution $\eta(\theta)$ is a Dirichlet distribution $\theta \sim \textup{Dir}(\rho_{1}, \ldots, \rho_{K})$ with prior hyperparameters $\rho_{k} > 0$, for $k = 1, \ldots, K$.

Unfortunately, the posterior distribution in \eqref{eq: posterior of theta given Y and S} is intractable. Therefore, we resort to approximate sampling approaches using MCMC. Below, we present two MCMC methods, namely SGLD and Gibbs sampling.

\subsection{Stochastic gradient Langevin dynamics} \label{sec: Stochastic gradient Langevin dynamics}

SGLD is an asymptotically exact gradient-based MCMC sampling approach that enables the use of subsamples of size $m \ll t$. A direct application of SGLD to generate samples for $\theta$ from the posterior distribution in \eqref{eq: posterior of theta given Y and S} is difficult. This is because $\theta$ lives in the probability simplex $\Delta$, which makes the task of keeping the iterates for $\theta$ inside $\Delta$ challenging. We overcome this problem by defining the surrogate variables $\phi_{1}, \ldots, \phi_{K}$ with
\[
\phi_{k} \overset{\text{ind.}}{\sim} \text{Gamma}(\rho_{k}, 1), \quad k = 1, \ldots, K,
\]
and the mapping from $\phi$ to $\theta$ as
\begin{equation} \label{eq: mapping from phi to theta}
\theta(\phi)_{k} := \frac{\phi_{k}}{\sum_{j = 1}^{K} \phi_{j}}, \quad k = 1, \ldots, K.
\end{equation}
It is well-known that the resulting $(\theta_{1}, \ldots, \theta_{K})$ has a Dirichlet distribution $\text{Dir}(\rho_{1}, \ldots, \rho_{K})$, which is exactly the prior distribution $\eta(\theta)$. Therefore, this change of variables preserves the originally constructed probabilistic model. Moreover, since $\phi = (\phi_{1}, \ldots, \phi_{K})$ takes values in $[0, \infty)^{K}$, we run SGLD for $\phi$, where the $j$'th update is 
\begin{equation} \label{eq: SGLD with phi}
\phi^{(j)} = \left\vert \phi^{(j-1)} + \frac{\gamma_{n}}{2} \left( \nabla_{\phi} \ln p(\phi^{(j-1)}) +\frac{n}{m} \sum_{i = 1}^{m} \nabla_{\phi} \ln p_{s_{u_{i}}, \epsilon}(y_{u_{i}} | \phi^{(j-1)}) \right)+ \gamma_{n} W_{j} \right\vert, \quad W_{j} \sim \mathcal{N}(0, I_{K}).
\end{equation}
where $u = (u_{1}, \ldots, u_{m})$ is a random subsample of $\{1, \ldots, n\}$. In \eqref{eq: SGLD with phi}, the `new' prior and likelihood functions are
\begin{equation} \label{eq: prior and likelihood wrt phi}
p(\phi) := \prod_{k = 1}^{K} \text{Gamma}(\phi_{k}; \alpha_{i}, 1), \quad  p_{s, \epsilon}(y | \phi) := h_{s, \epsilon}(y | \theta(\phi)).
\end{equation}
The reflection in \eqref{eq: SGLD with phi} via taking the component-wise absolute value is necessary because each $\phi_{k}^{(j)}$ must be positive. Step 3 of Algorithm \ref{alg: Differentially private online learning} can be approximated by running SGLD for some $M > 0$ iterations. To exploit the SGLD updates from the previous time, one should start the updates at time $n$ by setting the initial value for $\phi$ to the last SGLD iterate at time $n-1$.

The next proposition provides the explicit formulae for the gradients of the log-prior and the log-likelihood of $\phi$ in \eqref{eq: SGLD with phi}. A proof is given in Appendix \ref{sec: Proof for SGLD update}.
\begin{prop} \label{prop: gradient of the prior and likelihood of phi}
For $p(\phi)$ and in $p(y | \phi)$ in \eqref{eq: prior and likelihood wrt phi}, we have
\[
[\nabla_{\phi} \ln p(\phi)]_{i} = \frac{\alpha_{i}-1}{\phi_{i}} - 1, \quad [\nabla_{\phi} \ln p(y | \phi)]_{i} = \sum_{k = 1}^{K-1} J(i, k) \frac{g_{S, \epsilon}(y | k) - g_{S, \epsilon}(y | K)}{h_{S, \epsilon}(y | \theta(\phi))},
\]
where $J$ is a $K \times (K-1)$ Jacobian matrix whose $(i, j)$th element is
\[
J(i, j) = \mathbb{I}(i=j)\frac{1}{\sum_{k = 1}^{K} \phi_{k}} - \frac{\phi_{j}}{\left(\sum_{k = 1}^{K} \phi_{k} \right)^{2}}.
\]
\end{prop}

\subsection{Gibbs sampling} \label{sec: Gibbs sampling  }

An alternative to SGLD is the Gibbs sampler, which operates on the joint posterior distribution of $\theta$ and $X_{1:n}$ given $Y_{1:n} = y_{1:n}$ and $S_{1:n} = s_{1:n}$,
\begin{equation*}
p(\theta, x_{1:n} | y_{1:n}, s_{1:n}) \propto \eta(\theta) \left[ \prod_{t = 1}^{n} \theta_{x_{t}} g_{s_{t},\epsilon}(y_{t} | x_{t}) \right].
\end{equation*}
The full conditional distributions of $X_{1:n}$ and $\theta$ are tractable. Specifically, for $X_{1:n}$, we have
\begin{equation} \label{eq: full conditional of X given Y and theta}
p(x_{1:n} | y_{1:n}, s_{1:n}, \theta) = \prod_{t= 1}^{n} p_{s_{t}, \epsilon}(x_{t} | y_{t}, \theta),
\end{equation}
where $p_{s_{t}, \epsilon}(x_{t} | y_{t}, \theta)$ is defined in \eqref{eq: post of x given y and theta}. Therefore, \eqref{eq: full conditional of X given Y and theta} is a product of $n$ categorical distributions, each with support $[K]$. Furthermore, the full conditional distribution of $\theta$ is a Dirichlet distribution due to the conjugacy between the categorical and the Dirichlet distributions. Specifically, 
\[
p(\theta | x_{1:n}, y_{1:n}, s_{1:n}) = \text{Dir}(\theta  | \rho^{\text{post}}_{1}, \ldots, \rho^{\text{post}}_{K}),
\]
where the hyperparameters of the posterior distribution are given by $\rho^{\text{post}}_{k} := \rho_{k} +  \sum_{t = 1}^{n} \mathbb{I}(x_{t} = k)$ for $k = 1, \ldots, K$. 

Computational load at time $t$ of sampling from $t$ distributions in \eqref{eq: full conditional of X given Y and theta}, is proportional to $t K$, which renders the computational complexity of Gibbs sampling $\mathcal{O}(n^{2} K)$ after $n$ time steps. This can be computationally prohibitive when $n$ gets large. 

\section{Theoretical analysis} \label{sec: Theoretical analysis}

We address two questions concerning AdOBEst-LDP in Algorithm \ref{alg: Differentially private online learning} when it is run with RRRR whose subset is selected as described in Section \ref{sec: Subset selection for RRRR}. (i) Does the targeted posterior distribution based on the observations generated by Algorithm \ref{alg: Differentially private online learning} converge to the true value $\theta^{\ast}$?
(ii) How frequently does Algorithm \ref{alg: Differentially private online learning} with RRRR select the optimum subset $S$ according to the chosen utility function?

\subsection{Convergence of the posterior distribution} \label{sec: Convergence of the posterior distribution}
We begin by developing the joint probability distribution of the random variables involved in AdOBEst-LDP. 
\begin{itemize}
\item Given $Y_{1:n}$ and $S_{1:n}$, the posterior distribution $\Pi(\cdot | Y_{1:n}, S_{1:n})$ is defined such that for any measurable set $A \subseteq \Delta$, the posterior probability of $\{ \theta \in A\}$ is given by
\begin{equation} \label{eq: posterior pdf}
\Pi(A | Y_{1:n}, S_{1:n}) := \frac{\int_{A}  \eta(\theta) \prod_{t = 1}^{n} h_{S_{t}, \epsilon}(Y_{t}| \theta)  \mathrm{d}\theta}{\int_{\Delta} \eta(\theta) \prod_{t = 1}^{n} h_{S_{t}, \epsilon}( Y_{t} | \theta) \mathrm{d}\theta}.
\end{equation}
\item Let $Q(\cdot  | Y_{1:n}, S_{1:n}, \Theta_{n-1})$ be the probability distribution corresponding to the posterior sampling process for $\Theta_{n}$. Note that if exact posterior sampling were used, we would have $Q(A | Y_{1:n}, S_{1:n}, \Theta_{n-1}) = \Pi(A | Y_{1:n}, S_{1:n})$; however, when approximate sampling techniques are used to target $\Pi$, such as SGLD or Gibbs sampling, the equality does not hold in general.
\item For $\theta \in \Delta$, let
\[
S^{\ast}_{\theta} := \{ \sigma_{\theta}(1), \ldots, \sigma_{\theta}(k_{\theta}^{\ast}) \}, \quad \text{with} \quad k^{\ast}_{\theta} := \arg \max_{k \in \{0, \ldots, K-1 \}} U(\theta, S_{k, \theta}, \epsilon),
\]
be the best subset according to $\theta$, where $S_{k, \theta} = \{\sigma_{\theta}(1), \ldots, \sigma_{\theta}(k) \}$ is defined in \eqref{eq: Ss considered for RRRR}. Given $\Theta_{1:t-1}$ and $Y_{1:t}$, $S_{t}$ depends only on $\Theta_{t-1}$ and it is given by $S_{t} = S^{\ast}_{\Theta_{t-1}}$.
\end{itemize}
Combining all, the joint law of $S_{1:n}, Y_{1:n}$ can be expressed as 
\begin{equation} \label{eq: joint law of Y S and theta}
P_{\theta^{\ast}}(S_{1:n}, Y_{1:n}) :=  \prod_{t = 1}^{n}  h_{S_{t}, \epsilon}(Y_{t} | \theta^{\ast}) \left[ \int_{\Delta}\mathbb{I}(S_{t} = S_{k^{\ast},\theta_{t-1}})  Q( \mathrm{d}\theta_{t-1} | Y_{1:t-1}, S_{1:t-1}, \theta_{t-2}) \right],
\end{equation}
where we use the convention that $Q( \mathrm{d}\theta_{0} | Y_{1:0}, S_{1:0}, \theta_{-1})  = \delta_{\theta_{\text{init}}}(\mathrm{d}\theta_{0})$ for an initial value $\theta_{\text{init}} \in \Delta$.

The posterior probability in \eqref{eq: posterior pdf} is a random variable with respect to $P_{\theta^{\ast}}$ defined in \eqref{eq: joint law of Y S and theta}. Theorem \ref{thm: convergence to posterior} establishes that under the fairly mild Assumption \ref{asmp: prior has mass around theta} on the prior, the $\Pi(\cdot | Y_{1:n}, S_{1:n})$ converges to $\theta^{\ast}$ regardless of the choice of $Q$ for posterior sampling. 

\begin{asmp} \label{asmp: prior has mass around theta}
There exist finite positive constants $d > 0$ and $B > 0$ such that $\eta(\theta)/\eta(\theta') < B$ for all $\theta, \theta' \in \Delta$ whenever $\| \theta' - \theta^{\ast} \| < d$.
\end{asmp}
\begin{thm} \label{thm: convergence to posterior}
Under Assumption \ref{asmp: prior has mass around theta}, there exists a constant $c > 0$ such that, for any $0 < a < 1$ and the sequence of sets
\[
\Omega_{n} = \{ \theta \in \Delta: \|\theta - \theta^{\ast}\|^{2} \leq c n^{-a} \},
\]
the sequence of probabilities
\[
\lim_{n \rightarrow \infty} \Pi(\Omega_{n} | Y_{1:n}, S_{1:n}) \overset{P_{\theta^{\ast}}}{\rightarrow} 1,
\]
regardless of the choice of $Q$.
\end{thm}
A proof is given in Appendix \ref{appndx: Convergence of the posterior distribution}, where the constant $c$ in the sets $\Omega_{n}$ is explicitly given. 

\subsection{Selecting the best subset} \label{sec: Selecting the best subset}

Let $S^{\ast} := S^{\ast}_{\theta^{\ast}}$ be the best subset at $\theta^{\ast}$. In this part, we prove that if posterior sampling is performed exactly, the best subset is chosen with an expected long-run frequency of $1$. Our result relies on some mild assumptions.
\begin{asmp} \label{asmp: ordered theta}
The components of $\theta^{\ast}$ are strictly ordered, that is, $\theta_{\sigma_{\theta^{\ast}}(1)} > \ldots > \theta_{\sigma_{\theta^{\ast}}(K)}$.
\end{asmp}
\begin{asmp} \label{asmp: continuity of U}
Given any $S \subset [K]$ and $\epsilon > 0$, $U(\theta, S, \epsilon)$ is a continuous function of $\theta$ with respect to the $L_{2}$-norm.
\end{asmp}
\begin{asmp} \label{asmp: unique maximizer}
The solution of \eqref{eq: maximization of utility to choose S} is unique at $\theta^{\ast}$.
\end{asmp}
Assumption \ref{asmp: ordered theta} is required to avoid technical issues regarding the uniqueness of $S^{\ast}$. Assumptions \ref{asmp: continuity of U} and \ref{asmp: unique maximizer} impose a certain form of regularity on the utility function.
\begin{thm} \label{thm: choosing the best subset}
Suppose Assumptions \ref{asmp: prior has mass around theta}-\ref{asmp: unique maximizer} hold and $\Theta_{t}$s are generated by exact sampling, that is, $Q(A | Y_{1:t}, S_{1:t}) = \Pi(A | Y_{1:t}, S_{1:t})$ for all measurable $A \subseteq \Delta$. Then,
\begin{equation} \label{eq: convergence of probability}
\lim_{n \rightarrow \infty} P_{\theta^{\ast}}(S_{n} = S^{\ast}) \rightarrow 1.
\end{equation}
As a corollary, $S^{\ast}$ is selected with an expected long-run frequency of $1$, that is,
\begin{equation} \label{eq: sublinear regret}
\lim_{n \rightarrow \infty} \frac{1}{n} \sum_{t = 1}^{n} E_{\theta^{\ast}}\left[ \mathbb{I}(S_{t} = S^{\ast}) \right]= 1.
\end{equation}
\end{thm}
The result in \eqref{eq: sublinear regret} can be likened to sublinear regret from the reinforcement learning theory.

\section{Numerical results} \label{sec: Numerical Results}
We tested\footnote{The MATLAB code at \url{https://github.com/soneraydin/AdOBEst_LDP} can be used to reproduce the results obtained in this paper.} the performance of AdOBEst-LDP when the subset $S$ in RRRR is determined according to a utility function in Section \ref{sec: Subset selection for RRRR}. We compared AdOBEst-LDP when combined with each of the utility functions defined in Sections \ref{sec: Fisher Information Matrix}-\ref{sec: Probability of honest response} with its non-adaptive counterpart when SRR is used to generate $Y_{t}$ at all steps. We also included the semi-adaptive subset selection method in Section \ref{sec: Semi-adaptive approach} into the comparison. For the semi-adaptive approach, we obtained results for five different values of its $\alpha$ parameter, namely $\alpha \in \{0.2, 0.6, 0.8, 0.9, 0.95\}$. 

We ran each method for 50 Monte Carlo runs. Each run contained $T = 500 K$ time steps. For each run, the sensitive information is generated as $X_{t} \overset{\text{i.i.d.}}{\sim} \text{Cat}(\theta^{\ast})$ where $\theta^{\ast}$ itself was randomly drawn from $\text{Dirichlet}(\rho, \ldots, \rho)$. Here, the parameter $\rho$ was used to control the unevenness among the components of $\theta^{\ast}$. (Smaller $\rho$ leads to more uneven components in general). At each time step, Step 3 of Algorithm \ref{alg: Differentially private online learning} was performed by running $M = 20$ updates of an SGLD-based MCMC kernel as described Section \ref{sec: Stochastic gradient Langevin dynamics}. In SGLD, we took the subsample size $m = 50$ and the step-size parameter $a = \frac{0.5}{t}$ at time step $t$. Prior hyperparameters for the gamma distribution were taken $\rho_{0} = 1_{K}$. The posterior sample $\Theta_{t}$ was taken as the last iterate of those SGLD updates. Only for the last time step, $t = T$, the number of MCMC iterations was taken $2000$ to reliably calculate the final estimate $\hat{\theta}$ of $\theta$ by averaging the last $1000$ of those $2000$ iterates. (This average is the MCMC approximation of the posterior mean of $\theta$ given $Y_{1:T}$ and $S_{1:T}$.) We compared the mean posterior estimate of $\theta$ and the true value, and the performance measure was taken as the TV distance between $\text{Cat}(\theta^{\ast})$ and $\text{Cat}(\hat{\theta})$, that is, 
\begin{equation} \label{eq: error measure}
\frac{1}{2} \sum_{i = 1}^{K}  |\hat{\theta}_{i} - \theta_{i} |.
\end{equation}
Finally, the comparison among the methods was repeated for all the combinations ($K, \epsilon, \kappa, \rho$) of $K \in \{10, 20\}$, $\epsilon \in \{0.5, 1, 5\}$, $\kappa \in \{0.8, 0.9\}$, and $\rho \in \{0.01, 0.1, 1\}$. 

\begin{figure}[t!]
    \centerline{
    \begin{subfigure}[b]{12.5cm}
        \centerline{
        \includegraphics[width=0.50\textwidth]{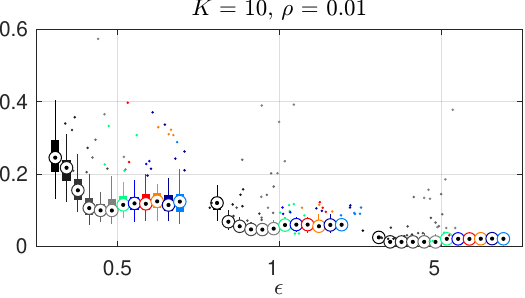}
        \includegraphics[width=0.50\textwidth]{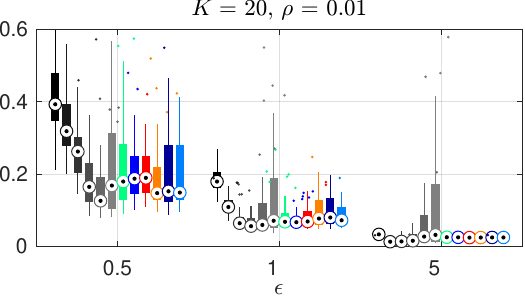}}
        \centerline{
        \includegraphics[width=0.50\textwidth]{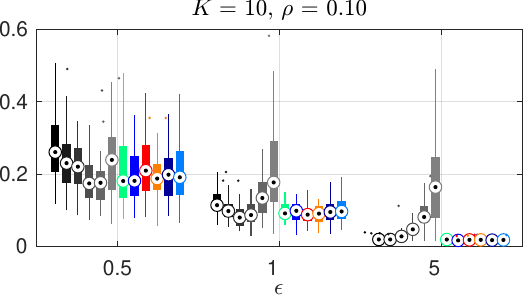}
        \includegraphics[width=0.50\textwidth]{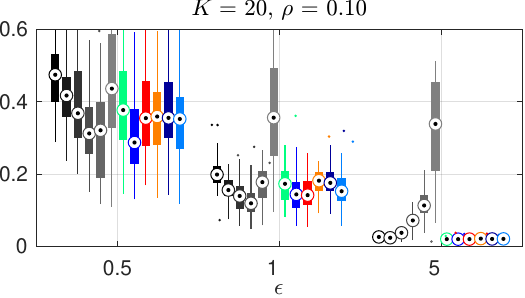}}
        \centerline{
        \includegraphics[width=0.50\textwidth]{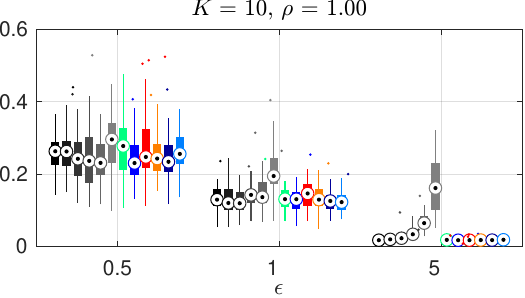}
        \includegraphics[width=0.50\textwidth]{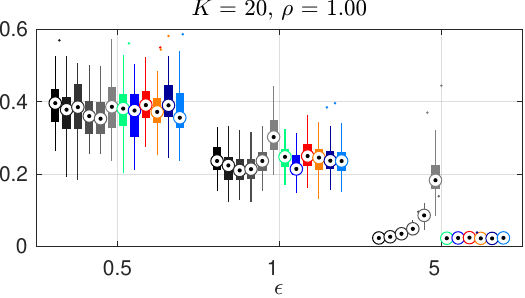}}
    \end{subfigure}%
    ~
    \begin{subfigure}[b]{3.5cm}
        \centering
        \includegraphics[width=\textwidth]{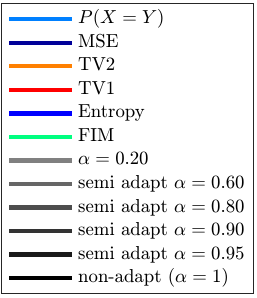}
    \end{subfigure}
    }
      \caption{TV distance in \eqref{eq: error measure} for $K \in \{10, 20\}$, $\epsilon_{1} = 0.8 \epsilon$}
\label{fig: tv results epscoeff_0.8}
\end{figure}

The accuracy results for the methods in comparison are summarized in Figures \ref{fig: tv results epscoeff_0.8} and \ref{fig: tv results epscoeff_0.9} in terms of the error given in \eqref{eq: error measure}. The box plots are centered at the error median, and the whiskers stretch from the minimum to the maximum over the 50 MC runs, excluding the outliers.  When the medians are compared, the fully adaptive algorithms, which use a utility function to select $S_{t}$, yield comparable results to the best semi-adaptive approach in both figures. As one may expect, the non-adaptive approach yielded the worst results in general, especially in the high-privacy regimes (smaller $\epsilon$) and uneven $\theta^{\ast}$ (smaller $\rho$). We also observe that, while most utility metrics are generally robust, the one based on FIM seems sensitive to the choice of $\epsilon_{1}$ parameter. This can be attributed to the fact that the FIM approaches singularity when $\epsilon_{2}$ is too small, which is the case if $\epsilon_{1}$ is chosen too close to $\epsilon$. Supporting this, we see that when $\epsilon_{1} = 0.8 \epsilon$, the utility metric based on FIM becomes more robust. Another remarkable observation is that the utility function based on the probability of honest response, $U_{6}$, has competitive performance despite being the lightest utility metric in computational complexity. Finally, while the semi-adaptive approach is computationally less demanding than most fully adaptive versions, the results show it can dramatically fail if its $\alpha$ hyperparameter is not tuned properly. In contrast, the fully adaptive approaches adapt well to $\epsilon$ or $\rho$ and do not need additional tuning.

\begin{figure}[t!]
    \centerline{
    \begin{subfigure}[b]{12.5cm}
        \centerline{
        \includegraphics[width=0.50\textwidth]{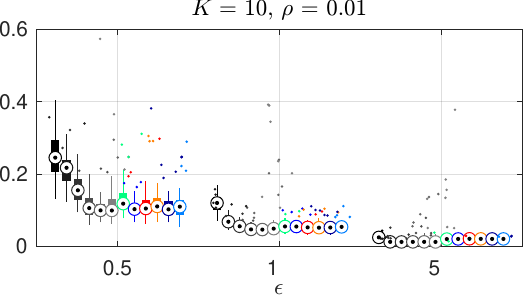}
        \includegraphics[width=0.50\textwidth]{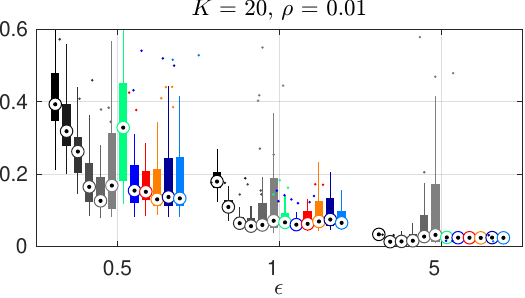}}
        \centerline{
        \includegraphics[width=0.50\textwidth]{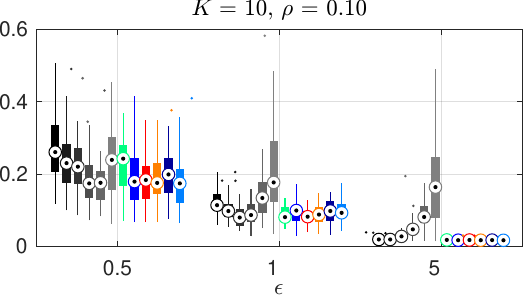}
        \includegraphics[width=0.50\textwidth]{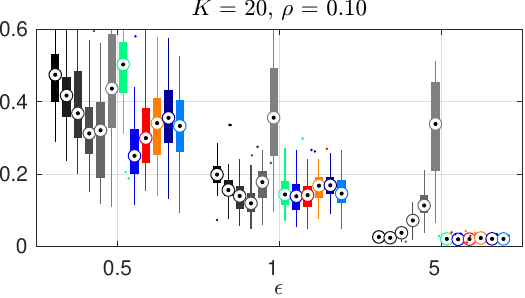}}
        \centerline{
        \includegraphics[width=0.50\textwidth]{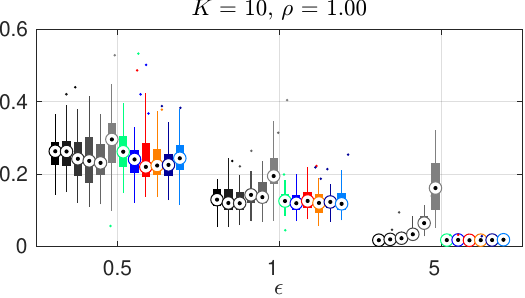}
        \includegraphics[width=0.50\textwidth]{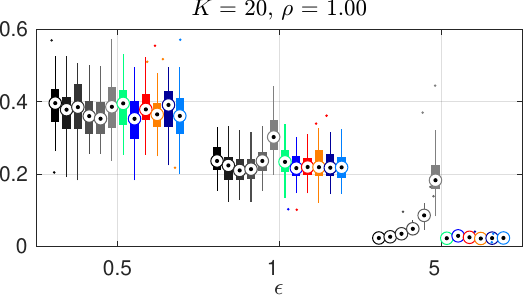}}
    \end{subfigure}
    ~
    \begin{subfigure}[t]{3.5cm}
        \centering
        \includegraphics[width=\textwidth]{legends_for_methods.pdf}
    \end{subfigure}
    }
      \caption{TV distance in \eqref{eq: error measure} for $K \in \{10, 20\}$, $\epsilon_{1} = 0.9 \epsilon$}
\label{fig: tv results epscoeff_0.9}
\end{figure}

In addition to the error graphs, the heat maps in Figures \ref{fig: heatmaps epscoeff_0.8} and \ref{fig: heatmaps epscoeff_0.9} show the effect of parameters $\rho$ and $\epsilon$ on the average cardinality of the subsets $S$ chosen by each algorithm (again, averaged over 50 Monte Carlo runs). According to these figures, increasing the value of $\rho$ causes an increase in the cardinalities of subsets chosen by each algorithm (except the nonadaptive one since it uses all $K$ categories rather than a smaller subset). This is expected since higher $\rho$ values cause $\text{Cat}(\theta^{\ast})$ to be closer to the uniform distribution, thus causing $X$ to be more evenly distributed among the categories. Moreover, for small $\rho$, increasing the value of $\epsilon$ causes a decrease in the cardinalities of these subsets, which can be attributed to a higher $\epsilon$, leading to a more accurate estimation. When we compare the utility functions for the adaptive approach among themselves, we observe that for $\epsilon_{1} = 0.8\epsilon$, the third utility function (TV1) uses the subsets with the largest cardinality (on average). However, when we increase the $\epsilon_{1}$ value to $\epsilon_{1} = 0.9\epsilon$, the second utility function (FIM) uses the subsets with the largest cardinality. This might be due to the sensitivity of the FIM-based utility function to the choice of $\epsilon_{1}$ parameter that we mentioned before, which affects the invertibility of the Fisher information matrix when $\epsilon_{1}$ is too close to $\epsilon$.

\begin{figure}[t]
  \centering
  \includegraphics[scale = 0.6]{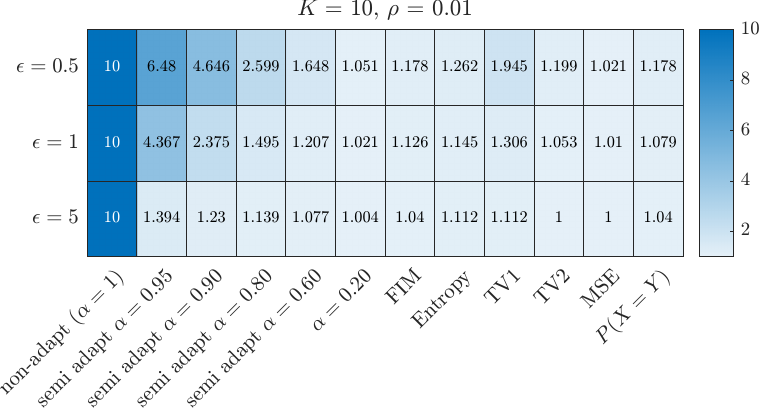} 
  \includegraphics[scale = 0.6]{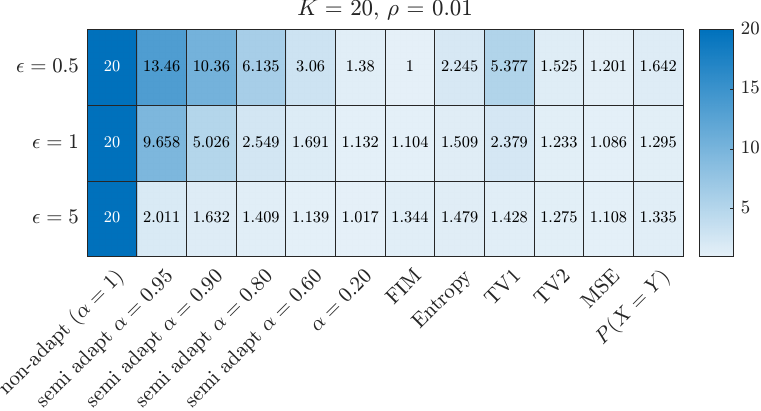} \\
  \includegraphics[scale = 0.6]{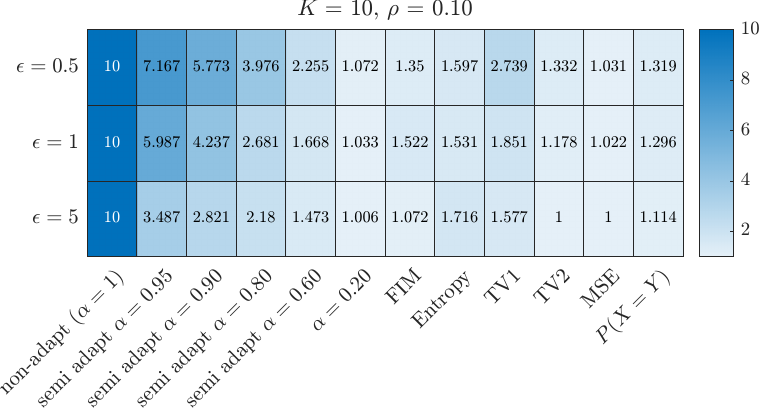} 
  \includegraphics[scale = 0.6]{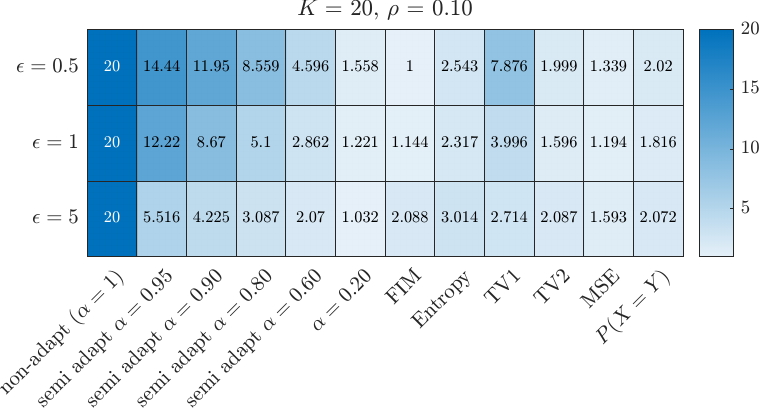} \\
  \includegraphics[scale = 0.6]{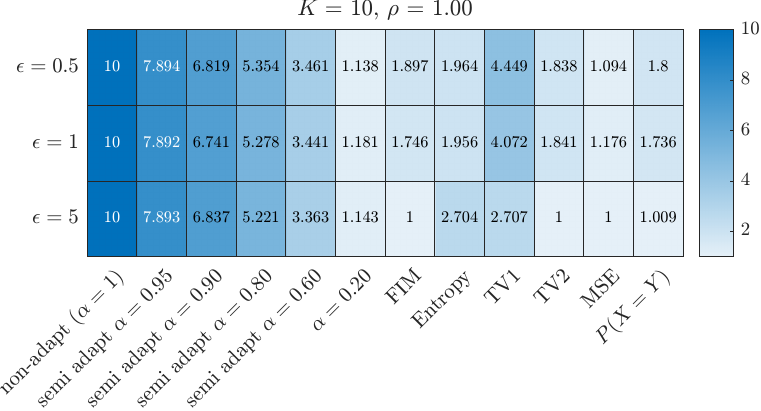}  
  \includegraphics[scale = 0.6]{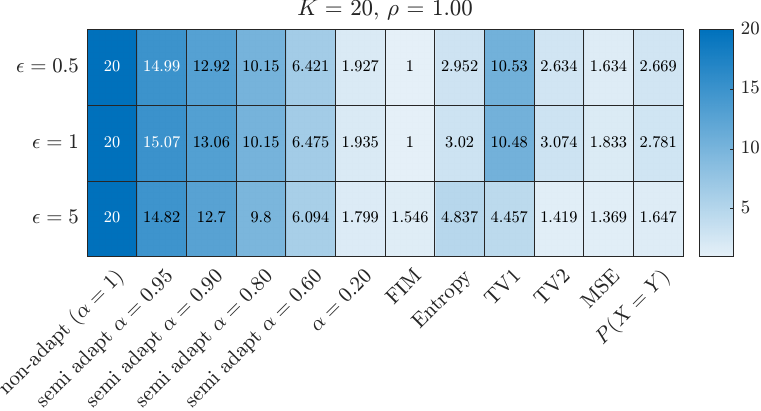}  
  \caption{Average cardinalities of the subsets selected by each method, for $K \in \{10, 20\}$, $\epsilon_{1} = 0.8 \epsilon$}
\label{fig: heatmaps epscoeff_0.8}
\end{figure}

\begin{figure}[t]
  \centering
  \includegraphics[scale = 0.6]{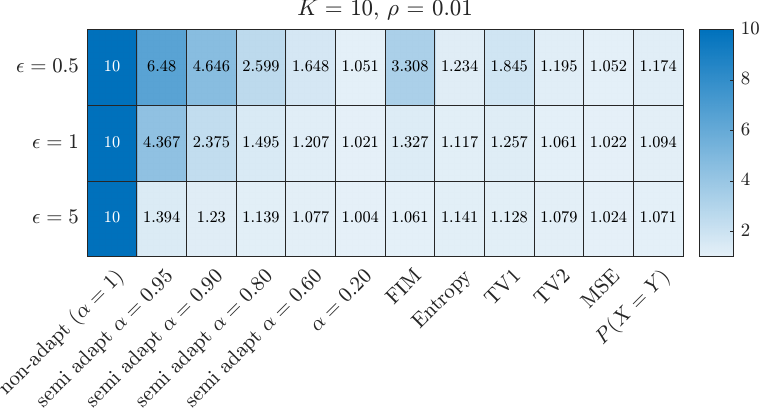} 
  \includegraphics[scale = 0.6]{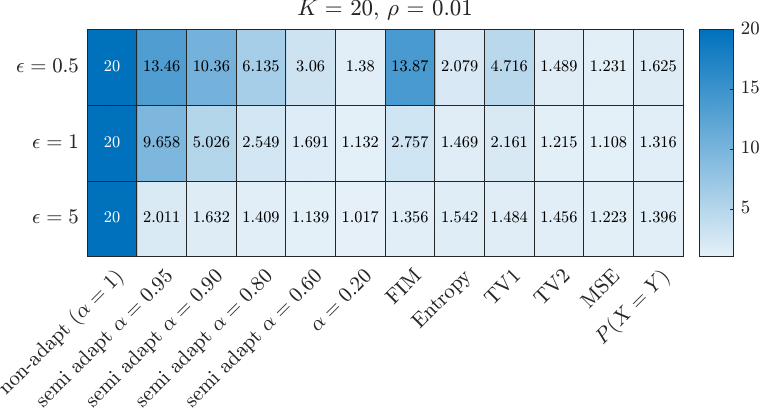} \\
  \includegraphics[scale = 0.6]{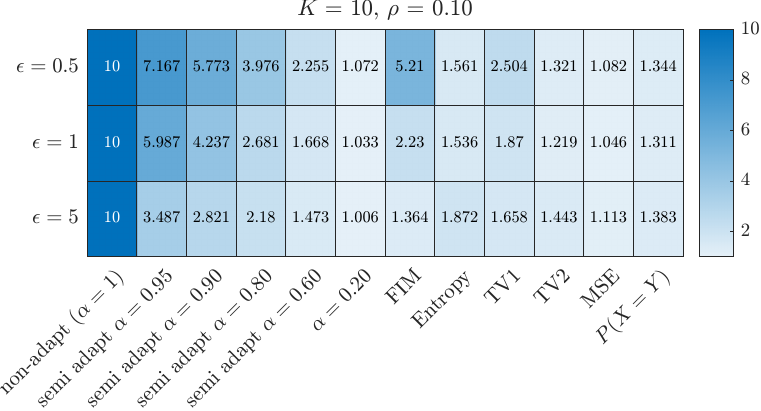} 
  \includegraphics[scale = 0.6]{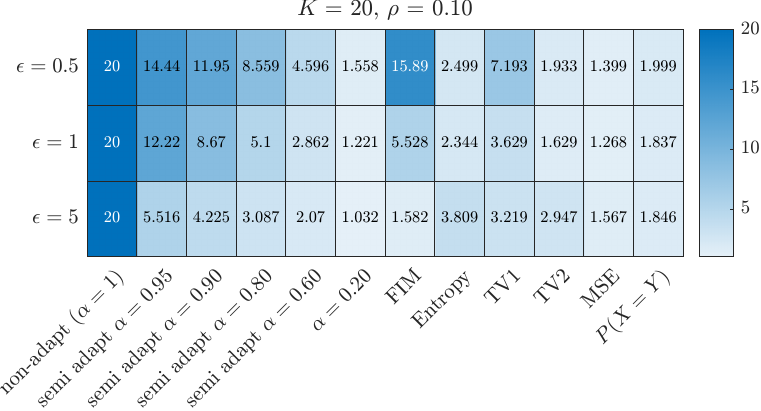} \\
  \includegraphics[scale = 0.6]{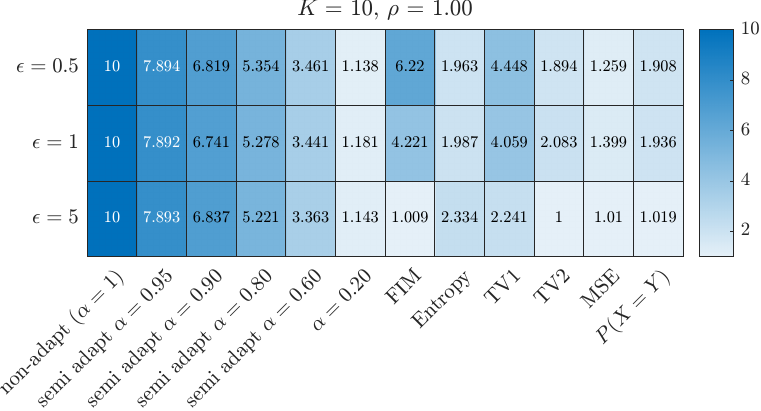}  
  \includegraphics[scale = 0.6]{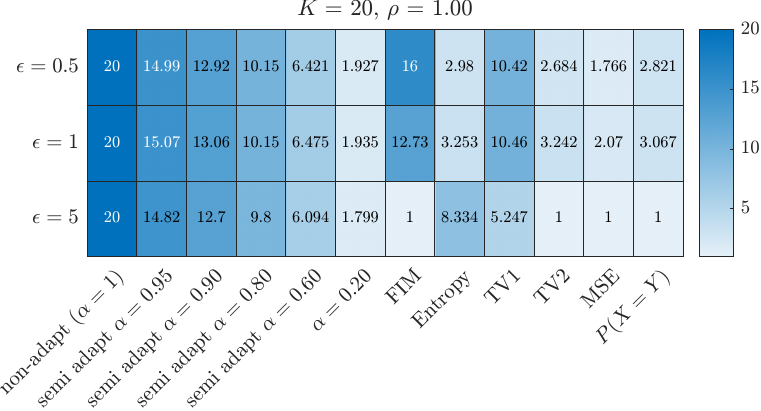}  
  \caption{Average cardinalities of the subsets selected by each method, for $K \in \{10, 20\}$, $\epsilon_{1} = 0.9 \epsilon$}
\label{fig: heatmaps epscoeff_0.9}
\end{figure}

\section{Conclusion} \label{sec: Conclusion}

In this paper, we proposed a new adaptive framework, AdOBEst-LDP, for online estimation of the distribution of categorical data under the $\epsilon$-LDP constraint. AdOBEst-LDP, run with RRRR for randomization, encompasses both privatization of the sensitive data and accurate Bayesian estimation of population parameters from privatized data in a dynamic way. Our privatization mechanism (RRRR) is distinguished from the baseline approach (SRR) in a way that it operates on a smaller subset of the sample space rather than the entire sample space. We employed an adaptive approach to dynamically adjust the subset at each iteration, based on the knowledge about $\theta^{\ast}$ obtained from the past data. The selection of these subsets was guided by various alternative utility functions that we used throughout the paper. For the posterior sampling of $\theta$ at each iteration, we employed an efficient SGLD-based sampling scheme on a constrained region, namely the $K$-dimensional probability simplex. We distinguished this scheme from Gibbs sampling, which uses all of the historical data and is not scalable to large datasets.

In the numerical experiments, we demonstrated that AdOBEst-LDP can estimate the population distribution more accurately than the non-adaptive approach under experimental settings with various privacy levels $\epsilon$ and degrees of evenness among the components of $\theta^{\ast}$. While the performance of AdOBEst-LDP is generally robust for all the utility functions considered in the paper, the utility function based on the probability of honest response can be preferred due to its much lower computational complexity than the other utility functions. Our experiments also showed that the accuracy of the adaptive approach is comparable to that of the semi-adaptive approach. However, the semi-adaptive approach requires adjusting its parameter $\alpha$ carefully, which makes it challenging to use.

In a theoretical analysis, we showed that, regardless of whether the posterior sampling is conducted exactly or approximately, the posterior distribution targeted in AdOBEst-LDP converges to the true population parameter $\theta^{\ast}$. We also showed that, under exact posterior sampling, the best subset given utility function is selected with probability $1$ in the long run.

It is important to note that the observations $\{Y_{t} \}_{t \geq 1}$ generated by AdOBEst-LDP are dependent. Therefore, the theoretical analysis presented in Section \ref{sec: Theoretical analysis} can also be seen as a contribution to the literature on the convergence of posterior distributions with dependent data. Additionally, we have already highlighted an analogy between AdOBEst-LDP and Thompson sampling \citep{Russo_et_al_2018}. Both methods involve posterior sampling, and the subset selection step in AdOBEst-LDP can be viewed as analogous to the action selection step in reinforcement learning schemes. In this regard, we believe that the theoretical results may also inspire future research on the convergence of dynamic reinforcement learning algorithms, especially those based on Thompson sampling.

Categorical distributions serve as useful non-parametric discrete approximations of continuous distributions. As a potential future direction, AdOBEst-LDP could be adapted for non-parametric density estimation. A key challenge in this context would be determining how to partition the support domain of the data.

RRRR is a practical LDP mechanism with a subset parameter that adapts based on past data. It has been shown to outperform SRR when leveraging the knowledge of $\theta^{\ast}$. However, in this work, it is \emph{not} proven that RRRR is the \emph{optimal} $\epsilon$-LDP mechanism with respect to the utility functions considered. While the optimal $\epsilon$-LDP mechanism could be identified numerically by solving a constrained optimization problem---where the utility function is maximized under the LDP constraint---it may not have a closed-form solution for complex utility functions. A promising direction for future research would be to compare the optimal $\epsilon$-LDP mechanism with the $\epsilon$-LDP RRRR mechanism by analyzing their transition probability matrices and assessing the suboptimality of RRRR. Additionally, insights from the optimal $\epsilon$-LDP mechanism could inspire the development of new, tractable, and approximately optimal $\epsilon$-LDP mechanisms.

\subsection*{Acknowledgements}
We thank our colleague Prof.\ Berrin Yanıkoğlu for reviewing the paper and providing insightful comments.

\bibliographystyle{ACM-Reference-Format}
\bibliography{references}

\appendix

\section{Proofs}

\subsection{Proofs for LDP of RRRR} \label{appndx: Proofs for LDP of RRRR}

\begin{proof}[Proof of Theorem \ref{thm: RRRR LDP}]
Let $k = |S|$. We can write as
\begin{equation} \label{eq: g explicit}
g_{S, \epsilon}(y | x) = 
\begin{cases} 
\frac{e^{\epsilon_{1}}}{e^{\epsilon_{1}} + k}  & x \in S, y \in S, x = y \\
\frac{1}{e^{\epsilon_{1}} + k }  & x \in S, y \in S, x \neq y \\
\frac{1}{K-k} \frac{1}{e^{\epsilon_{1}} + k }  & x \in S, y \notin S \\
\frac{1}{e^{\epsilon_{1}} + k }  & x \notin S, y \in S \\
\frac{e^{\epsilon_{2}}}{e^{\epsilon_{2}} + K-k-1 } \frac{e^{\epsilon_{1}}}{e^{\epsilon_{1}} + k} & x \notin S, y \notin S, x = y \\
\frac{1}{e^{\epsilon_{2}} + K-k-1 } \frac{e^{\epsilon_{1}}}{e^{\epsilon_{1}} + k} & x \notin S, y \notin S, x \neq y \\
\end{cases}.
\end{equation}
We will show that when $\epsilon_{1}, \epsilon_{2}$ are chosen according to the theorem, 
\begin{equation} \label{eq: condition for e-DP}
e^{-\epsilon} \leq \frac{g_{S, \epsilon}(y | x)}{g_{S, \epsilon}(y | x')} \leq e^{\epsilon}
\end{equation}
for all possible $x, x', y \in [K]$.  When $S = \emptyset$, the proof is trivial; we focus on the non-trivial case $S \neq \emptyset$. For the non-trivial case, the transition probability $g_{S, \epsilon}(y | x)$ requires checking the ratio in \eqref{eq: condition for e-DP} in $10$ different cases for $x, x', y$ concerning their interrelation. 
\begin{enumerate}[start=1, label={(\bfseries C\arabic*)}]
\item $x \in S$, $x' \notin S$, $y \in S$, $y = x$. We have
\[
\frac{g_{S, \epsilon}(y | x)}{g_{S, \epsilon}(y | x')} = \frac{\frac{e^{\epsilon_{1}}}{e^{\epsilon_{1}} + k}}{\frac{1}{e^{\epsilon_{1}} + k}} = e^{\epsilon_{1}}.
\]
Since $\epsilon_{1} \leq \epsilon$, \eqref{eq: condition for e-DP} holds.
\item $x \in S$, $x' \notin S$, $y \in S$, $y \neq x$. We have
\[
\frac{g_{S, \epsilon}(y | x)}{g_{S, \epsilon}(y | x')} = \frac{\frac{1}{e^{\epsilon_{1}} + k}}{\frac{1}{e^{\epsilon_{1}} + k}} = 1,
\]
which trivially implies \eqref{eq: condition for e-DP}.
\item $x \in S$, $x' \notin S$, $y \notin S$, $y = x'$. We need
\[
\frac{g_{S, \epsilon}(y | x)}{g_{S, \epsilon}(y | x')} = \frac{\frac{1}{K-k} \frac{1}{e^{\epsilon_{1}} + k}}{\frac{e^{\epsilon_{2}}}{e^{\epsilon_{2}} + K-k-1} \frac{e^{\epsilon_{1}}}{e^{\epsilon_{1}} + k}} = \frac{e^{\epsilon_{2} }+ K-k-1}{(K-k) e^{\epsilon_{1} + \epsilon_{2}} }.
\]
We can show that $\frac{g_{S, \epsilon}(y | x)}{g_{S, \epsilon}(y | x')} \leq 1 \leq e^{\epsilon}$ already holds since
\[
\frac{e^{\epsilon_{2} }+ K-k-1}{(K-k) e^{\epsilon_{1} + \epsilon_{2}} } = \frac{(K-k-1) + e^{\epsilon_{2}} }{(K-k-1) e^{\epsilon_{1}+ \epsilon_{2}} + e^{\epsilon_{1} + \epsilon_{2}}},
\]
and the first and the second terms in the numerator are smaller than those in the denominator, respectively. For the other side of the inequality, 
\[
\frac{e^{\epsilon_{2} }+ K-k-1}{(K-k) e^{\epsilon_{1} + \epsilon_{2}} } \geq e^{-\epsilon}
\]
requires
\[
e^{\epsilon_{2}} \leq  \frac{K-k-1}{e^{\epsilon_{1} - \epsilon} (K-k)-1}
\]
whenever $e^{\epsilon_{1} - \epsilon} (K-k)-1 > 0$, which is the condition given in the theorem.
\item $x \in S$, $x' \notin S$, $y \notin S$, $y \neq x'$. We need
\[
\frac{g_{S, \epsilon}(y | x)}{g_{S, \epsilon}(y | x')} = \frac{\frac{1}{K-k} \frac{1}{e^{\epsilon_{1}} +k}}{\frac{1}{e^{\epsilon_{2}} + K-k-1} \frac{e^{\epsilon_{1}}}{e^{\epsilon_{1}} +k}}=  \frac{e^{\epsilon_{2}} + (K-k) -1}{(K-k) e^{\epsilon_{1}}}.
\]
Since $\epsilon_{2} \leq \epsilon$, we have 
\[
e^{\epsilon_{2}} \leq (K-k) (e^{\epsilon + \epsilon_{1}} - 1 )+ 1.
\]
Hence
\[
\frac{g_{S, \epsilon}(y | x)}{g_{S, \epsilon}(y | x')} \leq \frac{ (K-k) (e^{\epsilon + \epsilon_{1}} - 1 )+ 1  + K-k -1}{(K-k)e^{\epsilon_{1}}} \leq e^{\epsilon},
\]
Hence, we proved the right-hand side inequality. For the left-hand side, we have
\[
\frac{e^{\epsilon_{2}} + K-k -1}{(K-k) e^{\epsilon_{1}}} = \frac{(e^{\epsilon_{2}}-1) + K-k}{(K-k) e^{\epsilon_{1}}} \geq \frac{K-k}{(K-k) e^{\epsilon_{1}}} = e^{-\epsilon_{1}} \geq e^{-\epsilon}
\]
since $\epsilon_{2} \geq 0$ and $\epsilon_{1} \leq \epsilon$.
\item $x, x' \in S$, $y \in S$, $y = x$. We have
\[
\frac{g_{S, \epsilon}(y | x)}{g_{S, \epsilon}(y | x')} = \frac{e^{\epsilon_{1}}/(e^{\epsilon_{1}} + k)}{1/(e^{\epsilon_{1}} + k)} = e^{\epsilon_{1}}.
\]
Since $\epsilon_{1} \leq \epsilon$, \eqref{eq: condition for e-DP} holds.
\item $x, x' \in S$, $y \in S$, $y \neq x$ and $y \neq x'$. We have
\[
\frac{g_{S, \epsilon}(y | x)}{g_{S, \epsilon}(y | x')} = \frac{1/(e^{\epsilon_{1}} + k)}{1/(e^{\epsilon_{1}} + k)} = 1.
\]
So \eqref{eq: condition for e-DP} trivially holds.
\item $x, x' \in S$, $y \notin S$. We have
\[
\frac{g_{S, \epsilon}(y | x)}{g_{S, \epsilon}(y | x')} = \frac{\frac{1}{K-k} \frac{1}{e^{\epsilon_{1}} + k } }{\frac{1}{K-k} \frac{1}{e^{\epsilon_{1}} + k } } = 1.
\]
So, \eqref{eq: condition for e-DP} trivially holds.
\item $x, x' \notin S$, $y \notin S$, $y = x$. We have
\[
\frac{g_{S, \epsilon}(y | x)}{g_{S, \epsilon}(y | x')}  = \frac{e^{\epsilon_{2}}/(e^{\epsilon_{2}} + K-k - 1) e^{\epsilon_{1}}/(e^{\epsilon_{1}} + k)}{1/(e^{\epsilon_{2}} + K-k - 1) e^{\epsilon_{1}}/(e^{\epsilon_{1}} + k)} = e^{\epsilon_{2}}.
\]
Since $\epsilon_{2} \leq \epsilon$, \eqref{eq: condition for e-DP} holds.
\item $x, x' \notin S$, $y \notin S$, $y \neq x$, $y \neq x'$.
We have
\[
\frac{g_{S, \epsilon}(y | x)}{g_{S, \epsilon}(y | x')}  = \frac{1/(e^{\epsilon_{2}} + K-k - 1) e^{\epsilon_{1}}/(e^{\epsilon_{1}} + k)}{1/(e^{\epsilon_{2}} + K-k - 1) e^{\epsilon_{1}}/(e^{\epsilon_{1}} + k)} = 1.
\]
So, \eqref{eq: condition for e-DP} trivially holds.
\item $x, x' \notin S$, $y \in S$. We have
\[
\frac{g_{S, \epsilon}(y | x)}{g_{S, \epsilon}(y | x')} = \frac{1/(e^{\epsilon_{1}} + |S|)}{1/(e^{\epsilon_{1}} + |S|)} = 1.
\]
So \eqref{eq: condition for e-DP} trivially holds.
\end{enumerate}
We conclude the proof by noting that any other case left out is symmetric in $(x, x')$ to one of the covered cases and, therefore, does not need to be checked separately.
\end{proof}

\subsection{Proofs about utility functions} \label{appndx: Proofs about utility functions}

\begin{proof}[Proof of Proposition \ref{prop: FIM of RRRR}]
Given $\theta \in \Delta$, let $\vartheta$ be the $(K-1) \times 1$ column vector such that $\vartheta_{i} = \theta_{i}$ for $i = 1, \ldots, K-1$. We can write the Fisher information matrix in terms of the score vector as follows.
\[
F(\theta; S, \epsilon) = \mathbb{E}_{Y} \left[ \nabla_{\vartheta} \ln h_{S, \epsilon}(Y| \theta) \nabla_{\vartheta} \ln h_{S, \epsilon}(Y | \theta)^{\top} \right] = \sum_{y = 1}^{K} h_{S, \epsilon}(y | \theta) \left[ \nabla_{\vartheta} \ln h_{S, \epsilon}(y|\theta) \nabla_{\vartheta} \ln h_{S, \epsilon}(y| \theta)^{\top} \right].
\]
Noting that 
\[
h_{S, \epsilon}(y | \theta) = \sum_{k =1}^{K-1} g_{S, \epsilon}(y | k) \vartheta_{k} + g_{S, \epsilon}(y | K) \left(1 - \sum_{k = 1}^{K-1}\vartheta_{k} \right),
\]
the score vector can be derived as
\begin{equation} \label{eq: score vector}
[\nabla_{\vartheta} \ln h_{S, \epsilon}(y | \theta)]_{k} = \frac{g_{S, \epsilon}(y | k) - g_{S, \epsilon}(y | K)}{h_{S, \epsilon}(y | \theta)}, \quad k = 1, \ldots, K-1.
\end{equation}
As the $K \times (K-1)$ matrix $A_{S, \epsilon}$ defined as $A(i, j) = g(i | j) - g(i | K)$, we can rewrite \eqref{eq: score vector} as $[\nabla_{\vartheta} \ln h_{S, \epsilon}(y | \theta)]_{k} = A_{S, \epsilon}(y, k)/h_{S, \epsilon}(y|\theta)$. Let $a_{y}$ be the $y$'th row of $A_{S, \epsilon}$, and recall that $D_{\theta}$ is defined as a diagonal matrix with $1/h_{S, \epsilon}(j | \theta)$  being the $j$'th element in the diagonal. Then, the Fisher information matrix is
\[
F(\theta; S, \epsilon) = \sum_{y =1}^{K} \frac{a_{y}^{\top}}{h_{S, \epsilon}(y | \theta)} \frac{a_{y}}{h_{S, \epsilon}(y | \theta)} h_{S, \epsilon}(y | \theta) =  \sum_{y =1}^{K} a_{y}^{\top} \frac{1}{h_{S, \epsilon}(y | \theta)} a_{y} = A_{S, \epsilon}^{\top} D_{\theta} A_{S, \epsilon},
\]
as claimed.
\end{proof}

Next, we prove that $F(\theta; S, \epsilon)$ is invertible. Let $G_{S, \epsilon}$ be the $K \times K$ matrix whose elements are
\begin{equation} \label{eq: G definition}
G_{S, \epsilon}(i, j) = g_{S, \epsilon}(i | j), \quad i, j = 1, \ldots, K.
\end{equation}
To prove that $F(\theta; S, \epsilon)$ is invertible, we first prove the intermediate result that $G_{S, \epsilon}$ is invertible.
\begin{lem} \label{lem: G is invertible}
$G_{S, \epsilon}$ is invertible for all $S \subset [K]$ and $\epsilon > 0$.
\end{lem}
\begin{proof}
It suffices to prove that of $G_{S, \epsilon}$ is invertible for $S = \{1,2, \ldots, k\}$ and for all $k \in \{0, \ldots, K-1\}$. For other $S$, $G_{S, \epsilon}$ can be obtained by permutation. Fix $k$ and let $S = \{1,2, \ldots, k\}$. It can be verified by inspection that $G_{S, \epsilon}$ is a block matrix as
\[
G_{S, \epsilon} = \begin{bmatrix} a_{1} I_{k} + a_{2} 1_{k} 1_{k}^{\top} & b 1_{k} 1_{K-k}^{\top} \\ c 1_{K-k}1_{k}^{\top} & d_{1} I_{K-k} + d_{2} 1_{K-k} 1_{K-k}^{\top} \end{bmatrix},
\]
where $I_{n}$ is the identity matrix of size $n$ and $1_{n}$ is the column vector of $1$'s of size $n$. The constants $a_{1}, a_{2}, b, c, d_{1}, d_{2}$ are given as
\begin{align*}
& a_{1} = \frac{e^{\epsilon_{1}}}{k+e^{\epsilon_{1}}} - a_{2}, \quad a_{2} = \frac{1}{k+e^{\epsilon_{1}}}, \quad b =  \frac{1}{k+e^{\epsilon_{1}}}, \quad c = \frac{1}{K-k} a_{2} \\
& d_{1} = \frac{e^{\epsilon_{2}}}{e^{\epsilon_{2}} + K-k-1} \frac{e^{\epsilon_{1}}}{k + e^{\epsilon_{1}}} - d_{2}, \quad d_{2} = \frac{1}{e^{\epsilon_{2}} + K-k-1} \frac{e^{\epsilon_{1}}}{k + \epsilon^{1}}.
\end{align*}
Also, note that since $\epsilon_{1} > 0$ and $\epsilon_{2} > 0$, $a_{1}$ and $d_{1}$ (whenever it is defined) are strictly positive.

The case $k = 0$ is trivial since then $G_{S, \epsilon} = d_{1}I_{K} + d_{2} 1_{K}1_{K}^{\top}$ is invertible. Hence, we focus on the case $0 < k < K$. For this case, firstly, note that the matrices on the diagonal are invertible. So, by Weinstein–Aronszajn identity, for $G_{S, \epsilon}$ to be invertible, it suffices to show that the matrix
\[
M = a_{1} I_{k} + a_{2} 1_{k} 1_{k}^{\top} - b 1_{k} 1_{K-k}^{\top} (d_{1} I_{K-k} + d_{2} 1_{K-k} 1_{K-k}^{\top} )^{-1} c 1_{K-k} 1_{k}^{\top}
\]
is invertible. Using the Woodbury matrix identity, the matrix $M$ can be expanded as
\begin{align*}
M &= a_{1} I_{k} + a_{2} 1_{k} 1_{k}^{\top} - b 1_{k} 1_{K-k}^{\top} \left( \frac{I_{K-k}}{d_{1}}  - \frac{1}{d_{1}} 1_{K-k} \left( \frac{1}{d_{2}} + 1_{K-k}^{\top} \frac{1}{d_{1}}  1_{K-k} \right)^{-1} 1_{K-k}^{\top} \frac{1}{d_{1}} \right) c 1_{K-k} 1_{k}^{\top} \\
&= a_{1} I_{k} + a_{2} 1_{k}1_{k}^{\top} - \frac{b c}{d_{1}} 1_{k} 1_{K-k}^{\top} 1_{K-k} 1_{k}^{\top} + \frac{bc}{d_{1}^{2}} \left( \frac{1}{d_{2}} + \frac{K-k}{d_{1}}\right)^{-1} 1_{k} 1_{K-k}^{\top} 1_{K-k} 1_{K-k}^{\top} 1_{K-k} 1_{k}^{\top} \\
& = a_{1} I_{k} + \left[ a_{2} - \frac{(K-k) b c}{d_{1}} + \frac{bc}{d_{1}^{2}} \left( \frac{1}{d_{2}} + \frac{K(K-k)}{d_{1}}\right)^{-1} \right] 1_{k} 1_{k}^{\top}.
\end{align*}
Inside the square brackets is a scalar, therefore, $M$ in question is the sum of an identity matrix and a rank-$1$ matrix, which is invertible. Hence, $G_{S, \epsilon}$ is invertible.
\end{proof}

\begin{proof}[Proof of Proposition \ref{prop: FIM is invertible}]
Note that $A_{S,\epsilon} = G_{S,\epsilon} J$, where the $K \times (K-1)$ matrix $J$ satisfies $J(i, i) = 1$ and $J(i, K) = -1$ for $i = 1, \ldots, K$, and $J(i, j) = 0$ otherwise.  Since $G_{S,\epsilon}$ is invertible, it is full rank. Also, the columns of $A_{S, \epsilon}$, denoted by $c^{A}_{i}$, $i = 1, \ldots, K-1$ are given by
\[
c^{A}_{1} = c^{G}_{1} - c^{G}_{K}, \quad \ldots, \quad c^{A}_{K-1} = c^{G}_{K-1} - c^{G}_{K},
\]
where $c^{G}_{i}$ is the $i$'th column of $G_{S, \epsilon}$ for $i = 1, \ldots, K$. Observe that $c^{A}_{i}$, $i = 1, \ldots, K-1$ are linearly independent since any linear combination of those columns is in the form of
\[
\sum_{i = 1}^{K-1} a_{i} c^{A}_{i} = \sum_{i = 1}^{K-1} a_{i} c^{G}_{i}  - \left(\sum_{i = 1}^{K-1} a_{i} \right) c^{G}_{K}.
\]
Since the columns of $G_{S, \epsilon}$ are linearly independent, the linear combination above becomes $0$ only if $a_{1} = \ldots = a_{K-1} = 0$. This shows that the columns of $A_{S, \epsilon}$ are also linearly independent. Thus, we conclude that $A_{S, \epsilon}$ has rank $K-1$. Finally, since $D_{\theta}$ is diagonal with positive diagonal entries, $A_{S, \epsilon}^{\top} D_{\theta} A_{S, \epsilon} = A_{S, \epsilon}^{\top} D_{\theta}^{1/2} D_{\theta}^{1/2} A_{S, \epsilon}$ is positive definite, hence invertible.
\end{proof}

The following proof contains a derivation of the utility function based on the MSE of the Bayesian estimator of $X$ given $Y$.
\begin{proof}[Proof of Proposition \ref{prop: utility MSE}]
It is well-known that the expectation in \eqref{eq: utility MSE} is minimized when $\widehat{e_{X}} =  \nu(Y) := \mathbb{E}_{\theta}[ e_{X} | Y]$, i.e. the posterior expectation of $e_{X}$ given $Y$. That is, 
\[
\min_{\widehat{e_{X}}} \mathbb{E}_{\theta} \left[ \| e_{X} -\widehat{e_{X}}(Y) \|^{2}  \right] = \mathbb{E}_{\theta} \left[ \| e_{X} -\nu(Y) \|^{2} \right].
\]
For the squared norm inside the expectation, we have
\begin{align}
\| e_{X} - \nu(Y) \|^{2} &= (1 - v(Y)_{X})^{2} + \sum_{k \neq X} v(Y)_{k}^{2} \nonumber \\
&= 1 + \nu(Y)_{X}^{2} - 2 \nu(Y)_{X} + \sum_{k \neq X} v(Y)_{k}^{2} \nonumber\\
&= 1 - 2 \nu(Y)_{X} + \sum_{k = 1}^{K} v(Y)_{k}^{2}. \label{eq: MSE expanded}
\end{align}
The expectation of the last term in \eqref{eq: MSE expanded} is
\begin{align}
\mathbb{E}_{\theta}\left[ \sum_{k = 1}^{K} v(Y)_{k}^{2} \right] &= \sum_{y} h_{S, \epsilon}(y | \theta) \sum_{x = 1}^{K} p_{S, \epsilon}(x | y, \theta)^{2} \nonumber  \\
&=\sum_{y = 1}^{K} \sum_{x = 1}^{K} h_{S, \epsilon}(y| \theta)  p_{S, \epsilon}(x | y, \theta)^{2} \nonumber \\
&= \sum_{y = 1}^{K} \sum_{x = 1}^{K} \frac{ g_{S, \epsilon}(y | x)^{2} \theta_{x}^{2}}{h_{S, \epsilon}(y | \theta)}.  \label{eq: last term of MSE expanded}
\end{align}
For the expectation of the second term in \eqref{eq: MSE expanded}, we have
\[
\mathbb{E}_{\theta}\left[ \nu(Y)_{X} \right] = \sum_{x, y} p_{S, \epsilon}(x | y, \theta) p_{S, \epsilon}(x, y | \theta),
\]
where $p(x, y | \theta)$ denotes the joint probability of $X, Y$ given $\theta$. Substituting $p(x, y | \theta) = p_{S, \epsilon}(x|y, \theta) h_{S, \epsilon}(y| \theta)$ into the equation above, we get
\begin{align}
\mathbb{E}_{\theta}\left[ \nu(Y)_{X} \right] &= \sum_{x = 1}^{K} \sum_{y = 1}^{K} h_{S, \epsilon}(y| \theta)  p_{S, \epsilon}(x | y, \theta)^{2}. \nonumber \\
& =  \sum_{x = 1}^{K} \sum_{y = 1}^{K} \frac{ g_{S, \epsilon}(y | x)^{2} \theta_{x}^{2}}{h_{S, \epsilon}(y | \theta)}, \label{eq: second term of MSE expanded}
\end{align}
which is equal to what we get in \eqref{eq: last term of MSE expanded}. Substituting \eqref{eq: last term of MSE expanded} and \eqref{eq: second term of MSE expanded} into \eqref{eq: MSE expanded}, we obtain
\begin{align*}
\mathbb{E}_{\theta} \left[ \| e_{X} - \nu(Y) \|^{2} \right] & = 1- 2 \sum_{x = 1}^{K} \sum_{y = 1}^{K} \frac{ g_{S, \epsilon}(y | x)^{2} \theta_{x}^{2}}{h_{S, \epsilon}(y | \theta)} + \sum_{x = 1}^{K} \sum_{y = 1}^{K} \frac{ g_{S, \epsilon}(y | x)^{2} \theta_{x}^{2}}{h_{S, \epsilon}(y | \theta)} \\ 
& = 1- \sum_{x = 1}^{K} \sum_{y = 1}^{K} \frac{ g_{S, \epsilon}(y | x)^{2} \theta_{x}^{2}}{h_{S, \epsilon}(y | \theta)}.
\end{align*}
Finally, using the definition $U_{5}(\theta, S, \epsilon) = -\min_{\widehat{e_{X}}} \mathbb{E}_{\theta} \left[ \| e_{X} -\widehat{e_{X}}(Y) \|^{2}  \right]$, we conclude the proof.
\end{proof}

\begin{proof}[Proof of Theorem \ref{thm: PXY global opt}]
The global maximization of $U_{6}$ over the set of all the subsets $S \subset [K]$ can be decomposed as 
\begin{align}
\max_{S \subset [K]} U_{6}(\theta, S, \epsilon) = \max_{k \in \{0, \ldots, K-1\}} \left\{ \max_{S \subset [K]: |S| = k}  U_{6}(\theta, S, \epsilon) \right\}. \label{eq: double maximization}
\end{align}
This inner maximization is equivalent to fixing the cardinality of $S$ to $k$ and finding the best $S$ with cardinality $k$. Now, the utility function can be written as
\begin{align*}
U_{6}(\theta, S, \epsilon) &= \frac{e^{\epsilon_{1}}}{e^{\epsilon_{1}} + k} \left( \sum_{i\in S} \theta_{i} + \frac{e^{\epsilon_{2}}}{e^{\epsilon_{2}} + K-k-1}  \sum_{i \notin S} \theta_{i} \right) \\
&=\left(\frac{e^{\epsilon_{1}}}{e^{\epsilon_{1}} + k}\right) \sum_{i\in S} \theta_{i} + \left(\frac{e^{\epsilon_{1}}}{e^{\epsilon_{1}} + k}\right)  \left( \frac{e^{\epsilon_{2}}}{e^{\epsilon_{2}} + K-k-1} \right)  \sum_{i \notin S} \theta_{i}, 
\end{align*} 
where $k$ appears in the first line since $|S| = k$. Note that $\sum_{i \in S} \theta_{i}$ and $\sum_{i \notin S} \theta_{i}$ sum to $1$ and the constants in front of the first sum is larger than that of the second. Hence, we seek to maximize an expression in the form of
\[
a x + b (1 - x)
\]
over a variable $x > 0$ when $a > b > 0$. This is maximized when $x > 0$ is taken as large as possible. Therefore, $U_{6}(\theta, S, \epsilon)$ is maximized when $\sum_{i \in S} \theta_{i}$ is made as large as possible under the constraint that $|S| = k$. Under this constraint this sum is maximized when $S$ has the indices of the $k$ largest components of $\theta$, that is, when $S = S_{k, \theta} = \{\sigma_{\theta}(1), \ldots, \sigma_{\theta}(k) \}$.

Then, \eqref{eq: double maximization} reduces to $\max_{k = 1, \ldots, K} U_{6}(\theta,  S_{k, \theta}, \epsilon)$. Hence, we conclude.
\end{proof}

\subsection{Proof for SGLD update} \label{sec: Proof for SGLD update}

\begin{proof}[Proof of Proposition \ref{prop: gradient of the prior and likelihood of phi}]
For the \textit{prior} component of the gradient, recall that we have $ \phi = (\phi_{1}, \ldots, \phi_{K})$, where 
\[
\phi_{i} \overset{\text{iid}}{\sim} \text{Gamma}(\rho_{i}, 1), \quad i = 1, \ldots, K.
\] 
Then, the marginal pdf of $\phi_{i}$ satisfies
\[
\ln p(\phi_{i}) = (\rho_{i} - 1) \ln \phi_{i} - \ln \Gamma(\rho_{i}) - \phi_{i}, \quad i = 1, \ldots, K.
\]
Taking the partial derivatives of $\ln p(\phi_{i})$ with respect to $\phi_{i}$, we have
\[
[\nabla_{\phi} \ln p(\phi)]_{i} = \frac{\rho_{i}-1}{\phi_{i}} - 1, \quad i = 1, \ldots, K.
\]
For the \textit{likelihood} component, given $\theta \in \Delta$, let the $(K-1) \times 1$ vector $\vartheta$ be the reparametrization of $\theta$ such that $\vartheta_{i} = \theta_{i}$ for $i = 1, \ldots, K-1$. Then, according to \eqref{eq: mapping from phi to theta},
\begin{equation} \label{eq: map from vartheta to phi}
\vartheta_{k} = \frac{\phi_{k}}{\sum_{j = 1}^{K} \phi_{j}}, \quad k = 1, \ldots, K-1.
\end{equation}
Using the chain rule, we can write the gradient of the log-likelihood with respect to $\phi$ as
\begin{align*}
\nabla_{\phi} \ln p(y | \phi) &= J \cdot \nabla_{\vartheta} \ln h_{S, \epsilon} (y | \theta).
\end{align*}
where $J$ is the $K \times (K-1)$ Jacobian matrix for the mapping from $\phi$ to $\vartheta$ in \eqref{eq: map from vartheta to phi}, whose $(i, j)$th element can be derived as 
\begin{align*}
J(i, j) &= \frac{\partial \vartheta_{j}}{\partial \phi_{i}} = \frac{\partial }{\partial \phi_{i}} \frac{\phi_{j}}{\sum_{k = 1}^{K} \phi_{k}} \\
& = \mathbb{I}(i=j)\frac{1}{\sum_{k = 1}^{K} \phi_{k}} - \frac{\phi_{j}}{\left(\sum_{k = 1}^{K} \phi_{k} \right)^{2}}.
\end{align*}
Using \eqref{eq: score vector} for $\nabla_{\vartheta} \ln h_{S, \epsilon}(y | \theta)$, we complete the proof.
\end{proof}

\subsection{Proofs for convergence and consistency results} \label{sec: Proofs for convergence and consistency results}
\subsubsection{Preliminary results} \label{sec: Preliminary results}

\begin{lem}\label{lem: bounds for Gtheta}
Given $\epsilon \geq 0$, there exists constants $0 < c \leq C < \infty$ such that for all $\theta \in \Delta$ and all $S \subset [K]$, we have
\[
c \leq g_{S, \epsilon}(y|x) \leq C, \quad c \leq h_{S, \epsilon}(y | \theta) \leq C.
\]
\end{lem}
\begin{proof}
The bounds for $g_{S, \epsilon}(y|x)$ can directly be verified from \eqref{eq: g explicit}. Moreover, 
\[
c \leq \min_{i = 1, \ldots, K} g_{S, \epsilon}(y | i) \leq h_{S, \epsilon}(y | \theta) = \sum_{i = 1}^{K} g_{S, \epsilon}(y | i) \theta_{i} \leq \max_{i = 1, \ldots, K} g_{S, \epsilon}(y | i) \leq C.
\]
Hence, we conclude.
\end{proof}

\begin{rem}
For the symbol $\theta$, which is used for $K \times 1$ probability vectors in $\Delta$, we will associate the symbol $\vartheta $ such that $\vartheta$ denotes the shortened vector of the first $K-1$ elements of $\theta$. Accordingly, we will use $\vartheta$, $\vartheta'$, $\vartheta^{\ast}$, etc, to denote the shortened versions of $\theta$, $\theta'$, $\theta^{\ast}$, etc.
\end{rem}

\begin{lem} \label{lem: inner product of gradient and difference}
For any $\theta, \theta' \in \Delta$, $y \in [K]$, and $S \subset [K]$, we have
\begin{align} \label{eq: gradian times difference}
\nabla_{\vartheta} \ln h_{S, \epsilon}(y | \theta)^{\top} (\vartheta - \vartheta') &= \frac{h_{S, \epsilon}(y | \theta) - h_{S, \epsilon}(y | \theta')}{h_{S, \epsilon}(y | \theta)}.
\end{align}
\end{lem}

\begin{proof}
Recall from \eqref{eq: score vector} that
\[
[\nabla_{\vartheta} \ln h_{S, \epsilon}(y | \theta)]_{i} = \frac{g_{S, \epsilon}(y | i) - g_{S, \epsilon}(y | K)}{h_{S, \epsilon}(y | \theta)},\quad i = 1, \ldots, K-1.
\]
Hence, 
\begin{align*}
\nabla_{\vartheta} \ln h_{S, \epsilon}(y | \theta)^{\top} \vartheta &= \frac{1}{h_{S, \epsilon}(y | \theta)} \sum_{i = 1}^{K-1} (g_{S, \epsilon}(y | i) - g_{S, \epsilon}(y | K)) \vartheta_{i} \\
&= \frac{1}{h_{S, \epsilon}(y|\theta)}  \left[ \sum_{i = 1}^{K-1} g_{S, \epsilon}(y | i) \vartheta_{i} - g_{S, \epsilon}(y | K) \sum_{i = 1}^{K-1} \vartheta_{i} \right]\\
& = \frac{1}{h_{S, \epsilon}(y|\theta)} \left[   \sum_{i = 1}^{K-1} g_{S, \epsilon}(y, i) \theta_{i} - g_{S, \epsilon}(y | K) (1 - \theta_{K}) \right]\\
&= \frac{1}{h_{S, \epsilon}(y|\theta)}  \left[ h_{S, \epsilon}(y | \theta) - g_{S, \epsilon}(y | K) \right],
\end{align*}
and likewise $\nabla_{\vartheta} \ln h_{S, \epsilon}(y | \theta)^{\top} \vartheta' = \frac{1}{h_{S, \epsilon}(y|\theta)}  \left[ h_{S, \epsilon}(y | \theta') - g_{S, \epsilon}(y | K) \right]$. Taking the difference between $\nabla_{\vartheta} \ln h_{S, \epsilon}(y | \theta)^{\top} \vartheta$ and $\nabla_{\vartheta} \ln h_{S, \epsilon}(y | \theta)^{\top} \vartheta'$, we arrive at the result.
\end{proof}

\paragraph{Concavity of $\ln h_{S, \epsilon}(y | \theta)$:} The following lemmas help with proving the concavity of $h_{S, \epsilon}(y | \theta)$ as a function of $\theta$. 
\begin{lem} \label{lem: inverse hyperbolic tangent for probabilities}
For $0 < b \leq a \leq 1$, we have $\ln \frac{a}{b} \geq \frac{a-b}{a} + \frac{(a-b)^{2}}{2}$.
\end{lem}
\begin{proof}
For $z > 0$, using the series based on the inverse hyperbolic tangent function, we can write 
\[
\ln z = 2 \sum_{k = 0}^{\infty} \frac{1}{2 k + 1} \left( \frac{z-1}{z+1} \right)^{2k + 1}.
\]
Apply the expansion to $z =a/b$ when $a \geq b$. Noting that $(z-1)/(z+1) = (a-b)/(a+b)$,
\begin{align*}
\ln \frac{a}{b} & \geq 2 \frac{a-b}{a+b} \geq \frac{2(a-b)}{2a } = \frac{a-b}{a} 
\end{align*}
where the difference is 
\[
\frac{2 (a-b)}{a+b} - \frac{a-b}{a} = \frac{2 a^{2} - 2 ab - a^{2} + b^{2}}{a(a+b)} = \frac{(a-b)^{2}}{(a+b)a} \geq \frac{(a-b)^{2}}{2}
\]
since $0 < a, b \leq 1$.
\end{proof}

\begin{lem} \label{eq: power ineq 1}
Let $0 < \alpha \leq 1$. For $x \geq 1$, we have $\frac{1}{\alpha} (x^{\alpha} - 1) \leq (x - 1)$.
\end{lem}
\begin{proof}
Consider $f(x) = \frac{1}{\alpha} (x^{\alpha} - 1) - (x - 1)$. We have $f(1) = 0$ and $f'(x) = x^{\alpha - 1} - 1 \leq 0$ for $x \geq 1$. Hence, we conclude.
\end{proof}

\begin{lem} \label{eq: power ineq 2}
Let $0 < a \leq b \leq 1$ and $0 < \alpha \leq 1$. Then, $b^{\alpha} - a^{\alpha} \geq \alpha (b-a)$.
\end{lem}

\begin{proof}
Fix $a$ and let $b = a + x$ for $0 \leq x \leq 1-a$. Consider the function $f(x) = (a+x)^{\alpha} - a^{\alpha} - \alpha x$. We have $f(0) = 0$ and $f'(x) = \alpha (a + x)^{\alpha - 1} - \alpha \geq 0$ over $0 \le x \le (1-a)$ since $a + x \leq 1$ and $\alpha - 1 \leq 0$. Hence, we conclude.
\end{proof}

\begin{lem} \label{lem: concavity}
Given $\epsilon > 0$, there exists $m_{0} > 0$ such that,  for all $S \subset [K]$ and $y \in [K]$, $\ln h_{S, \epsilon}(y | \theta)$ is a concave function of $\vartheta$ that satisfies
\[
\ln h_{S, \epsilon}(y | \theta) - \ln h_{S, \epsilon}(y | \theta') \geq \nabla_{\vartheta} \ln h_{S, \epsilon}(y | \theta)^{\top} (\vartheta - \vartheta') + m_{0} (h_{S, \epsilon}(y| \theta) -  h_{S, \epsilon}(y| \theta'))^{2}
\]
for all $\theta, \theta' \in \Delta$.
\end{lem}

\begin{proof}
We will look at the cases $h_{S, \epsilon}(y |  \theta) \geq h_{S, \epsilon}(y | \theta')$ and $h_{S, \epsilon}(y |  \theta) \leq h_{S, \epsilon}(y | \theta')$ separately.
\begin{enumerate}
\item Assume $h_{S, \epsilon}(y |  \theta) \geq h_{S, \epsilon}(y | \theta')$. Using Lemma \ref{lem: inverse hyperbolic tangent for probabilities}, we have
\begin{align*}
\ln h_{S, \epsilon}(y | \theta) - \ln h_{S, \epsilon}(y | \theta') & \geq  \frac{h_{S, \epsilon}(y|\theta) - h_{S, \epsilon}(y|\theta')}{h_{S, \epsilon}(y|\theta)} + \frac{1}{2} (h_{S, \epsilon}(y| \theta) -  h_{S, \epsilon}(y| \theta'))^{2} \\
&= \nabla_{\vartheta} \ln h_{S, \epsilon}(y | \theta)^{\top} (\vartheta - \vartheta')  +\frac{1}{2} (h_{S, \epsilon}(y| \theta) -  h_{S, \epsilon}(y| \theta'))^{2},
\end{align*}
where the last line follows from Lemma \ref{lem: inner product of gradient and difference}.
\item Assume $h_{S, \epsilon}(y|\theta) \leq h_{S, \epsilon}(y| \theta')$. Let $a = h_{S, \epsilon}(y | \theta)$ and $b =h_{S, \epsilon}(y|\theta')$. Let
\[
\alpha = \min \left\{ 1, \frac{\ln 2}{\ln (C/c)} \right\},
\]
where $c$ and $C$ are given in Lemma \ref{lem: bounds for Gtheta}. This $\alpha$ ensures that $0 < \alpha \leq 1$ and $b^{\alpha}/a^{\alpha} \leq 2$, so that we can use Taylor's expansion of $\ln \frac{b^{\alpha}}{a^{\alpha}}$ around $1$ and have
\[
\ln \frac{b}{a} =  \frac{1}{\alpha} \ln \frac{b^{\alpha}}{a^{\alpha}}= \frac{1}{\alpha } \left[ \sum_{k =1}^{\infty} (-1)^{k+1} \frac{1}{k}\left( \frac{b^{\alpha}}{a^{\alpha}} -1 \right)^{k} \right].
\]
Approximating the expansion up to its third term, we have the following upper bound on $\ln \frac{b}{a}$ as
\begin{equation} \label{eq: Taylor expansion upto 3}
\ln \frac{b}{a} \leq \frac{1}{\alpha} \left( \frac{b^{\alpha}}{a^{\alpha}} -1 \right) - \frac{1}{2 \alpha}  \left(  \frac{b^{\alpha}}{a^{\alpha}} - 1 \right)^{2} + \frac{1}{3 \alpha} \left( \frac{b^{\alpha}}{a^{\alpha}} - 1 \right)^{3}.
\end{equation}
For the third term, we have
\[
\frac{1}{3 \alpha} \left( \frac{b^{\alpha}}{a^{\alpha}} - 1 \right)^{3} = \frac{1}{3 \alpha} \left( \frac{b^{\alpha}}{a^{\alpha}} - 1 \right)^{2} \left( \frac{b^{\alpha}}{a^{\alpha}} - 1 \right) \leq \frac{1}{3 \alpha} \left( \frac{b^{\alpha}}{a^{\alpha}} - 1 \right)^{2},
\]
since $1 \leq \frac{b^{\alpha}}{a^{\alpha}} \leq 2$. Substituting this into \eqref{eq: Taylor expansion upto 3}, the inequality can be continued as
\begin{align} \label{eq: Taylor reduced to 2}
\ln \frac{b}{a} \leq \frac{1}{\alpha} \left( \frac{b^{\alpha}}{a^{\alpha}} -1 \right) - \frac{1}{6 \alpha}  \left(  \frac{b^{\alpha}}{a^{\alpha}} - 1 \right)^{2}.
\end{align}
Using Lemma \ref{eq: power ineq 1} to bound the first term in \eqref{eq: Taylor reduced to 2}, we have
\begin{align}
\ln \frac{b}{a} \leq \left( \frac{b}{a} -1 \right) - \frac{1}{6 \alpha}  \left(  \frac{b^{\alpha}}{a^{\alpha}} - 1 \right)^{2}.
\end{align}
Finally, using Lemma \ref{eq: power ineq 2} we can lower-bound the second term in \eqref{eq: Taylor reduced to 2} as 
\[
 \frac{b^{\alpha}}{a^{\alpha}} - 1  = \frac{b^{\alpha} - a^{\alpha}}{a^{\alpha}} \geq \frac{\alpha (b - a)}{a^{\alpha}} \geq \alpha (b-a),
\]
where the last inequality follows from $a \leq 1$ and $\alpha > 0$. We end up with 
\begin{align*}
\ln \frac{b}{a} \leq \left( \frac{b}{a} -1 \right) - \frac{\alpha}{6}  \left( b - a \right)^{2}.
\end{align*}
Referring to the definitions of $a$ and $b$, we have
\[
\ln h_{S, \epsilon}(y|\theta') - \ln h_{S, \epsilon}(y|\theta) \leq \left( \frac{ h_{S, \epsilon}(y|\theta')}{ h_{S, \epsilon}(y|\theta) } - 1 \right) - \frac{\alpha}{6}( h_{S, \epsilon}(y|\theta) - h_{S, \epsilon}(y|\theta'))^{2},
\]
or, reversing the inequality,
\begin{align} \label{eq: inequality for strong concavity}
\ln h_{S, \epsilon}(y|\theta) - \ln h_{S, \epsilon}(y|\theta') \geq \left(1 - \frac{ h_{S, \epsilon}(y|\theta')}{ h_{S, \epsilon}(y|\theta) } \right) + \frac{\alpha}{6}( h_{S, \epsilon}(y|\theta) - h_{S, \epsilon}(y|\theta'))^{2}.
\end{align}
Using \eqref{eq: gradian times difference} in Lemma \ref{lem: inner product of gradient and difference}, we rewrite \eqref{eq: inequality for strong concavity} as 
\begin{align} \label{eq: inequality for strong concavity - final}
\ln h_{S, \epsilon}(y|\theta) - \ln h_{S, \epsilon}(y|\theta') \geq  \nabla_{\vartheta} \ln h_{S, \epsilon}(y | \theta)^{\top} (\vartheta - \vartheta') + \frac{\alpha }{6} ( h_{S, \epsilon}(y|\theta) - h_{S, \epsilon}(y|\theta'))^{2},
\end{align}
which is the inequality we look for.
\end{enumerate}
To cover both cases, take $m_{0} = \min\{\frac{\alpha}{6}, \frac{1}{2} \}$. So, the proof is complete.
\end{proof}

Recalling that $S_{t}$ is the selected subset at time $t$, define
\begin{equation} \label{eq: Ut definition}
V_{t}(\theta, \theta') := ( h_{S_{t}, \epsilon}(Y_{i} | \theta) - h_{S_{t}, \epsilon}(Y_{t} | \theta'))^{2}.
\end{equation}
The proof of Theorem \ref{thm: convergence to posterior} requires a probabilistic bound for $\sum_{t = 1}^{n} V_{t}(\theta, \theta')$, which we provide next.

\begin{lem} \label{lem: Upperbound for U}
For all $\theta, \theta' \in \Delta$ and $t \geq 0$, $V_{t}(\theta, \theta') \leq \|\theta - \theta' \|^{2}$.
\end{lem}
\begin{proof}
Let the $1 \times K$ vector $r_{i}$ be the $i$th row of $G_{S_{t}, \epsilon}$. Then, we obtain
\begin{align*}
V_{t}(\theta, \theta') = [h_{S_{t}, \epsilon}(y | \theta) - h_{S_{t}, \epsilon}(y | \theta')]^{2} &= (r_{i}(\theta - \theta'))^{2} \nonumber \\
& = (\theta - \theta')^{\top} r_{i}^{\top} r_{i} (\theta - \theta') \\
&\leq  \| \theta - \theta' \|^{2},
\end{align*}
since every element of $r_{i}$ is at most $1$.
\end{proof}

\begin{lem} \label{lem: Lower bound for EU}
For all $\theta, \theta' \in \Delta$ and $t \geq 1$, there exists a constant $c_{u} > 0$ such that
\[
E_{\theta^{\ast}}[V_{t}(\theta, \theta')] \geq c_{u} \| \theta - \theta' \|^{2},
\]
where $E_{\theta^{\ast}}$ is the expectation operator with respect to $P_{\theta^{\ast}}$ defined in \eqref{eq: joint law of Y S and theta}, $\lambda_{\textup{min}}(A)$ is the minimum eigenvalue of the square matrix $A$ and the matrix $G_{S, \epsilon}$ is defined in \eqref{eq: G definition}.
\end{lem}

\begin{proof}
The overall expectation can be written as 
\begin{align}
E_{\theta^{\ast}}[V_{t}(\theta, \theta')] &= \sum_{S\subset [K]} P_{\theta^{\ast}}(S_{t} = S) E_{\theta^{\ast}}[V_{t}(\theta, \theta') | S_{t} = S] \nonumber
\end{align}
where the conditional expectation can be bounded as 
\begin{align}
E_{\theta^{\ast}}[V_{t}(\theta, \theta') | S_{t} = S] &= \sum_{i = 1}^{K} (h_{S, \epsilon}(i | \theta) - h_{S, \epsilon}(i | \theta'))^{2} h_{S, \epsilon}(i |\theta^{\ast}) \nonumber \\
& \geq \sum_{i = 1}^{K} (h_{S, \epsilon}(i | \theta) - h_{S, \epsilon}(i | \theta'))^{2} c, \label{eq: lower bound part 1}
\end{align}
where the second line is due to Lemma \ref{lem: bounds for Gtheta}. Further, let the $1 \times K$ vector $r_{i}$ be the $i$th row of $G_{S, \epsilon}$. Then,
\begin{align}
\sum_{i = 1}^{K} (h_{S, \epsilon}(i | \theta) - h_{S, \epsilon}(i | \theta'))^{2} &= \sum_{i = 1}^{K} ( r_{i}  (\theta - \theta'))^{2}\nonumber \\
&= \sum_{i = 1}^{K}  (\theta - \theta')^{\top} r_{i}^{\top} r_{i} (\theta - \theta') \nonumber \\
&= (\theta - \theta')^{\top} G_{S, \epsilon}^{\top} G_{S, \epsilon} (\theta - \theta') \nonumber \\
&= \| G_{S, \epsilon}(\theta - \theta') \|^{2} \nonumber \\
&\geq  \lambda_{\min}(G_{S, \epsilon}^{\top} G_{S, \epsilon}) \| \theta - \theta' \|^{2}. \label{eq: lower bound part 2}
\end{align}
Combining \eqref{eq: lower bound part 1} and \eqref{eq: lower bound part 2} and letting $c_{u} := c \min_{S \subset [K]} \lambda_{\min}(G_{S, \epsilon}^{\top} G_{S, \epsilon})$, we have
\begin{equation} \label{eq: lower bound for EUgivenS}
E_{\theta^{\ast}}[V_{t}(\theta, \theta') | S_{t} = S] \geq c_{u} \| \theta - \theta' \|^{2}.
\end{equation}
for all $S$, which directly implies $E_{\theta^{\ast}}[V_{t}(\theta, \theta')] \geq c_{u} \| \theta - \theta' \|^{2}$ for the overall expectation. Finally, $c_{u} > 0$ since, by Lemma \ref{lem: G is invertible}, every $G_{S, \epsilon}$ is invertible.
\end{proof}

Further, define
\[
W_{t}(\theta, \theta') := 1 - \frac{V_{t}(\theta, \theta')}{\| \theta - \theta'\|^{2}}.
\]
\begin{lem} \label{lem: bound on the expectation of product of V}
For any $0 < t_{1} < \ldots < t_{k}$, we have $E_{\theta^{\ast}} \left[ \prod_{i =1}^{k} W_{t_{i}}(\theta, \theta') \right] \leq (1 - c_{u} )^{k}$.
\end{lem} 
\begin{proof}
For simplicity, we drop $(\theta, \theta')$ from the notation and denote the random variables in question as $W_{t_{1}}, \ldots, W_{t_{k}}$. We can write 
\begin{align}
 E_{\theta^{\ast}} \left( \prod_{i =1}^{k} W_{t_{i}}  \right) &= E_{\theta^{\ast}} \left[ E_{\theta^{\ast}} \left( \left. \prod_{i =1}^{k} W_{t_{i}} \right\vert W_{t_{1}}, \ldots, W_{t_{k-1}} \right) \right] \nonumber \\
& = E_{\theta^{\ast}} \left[ \left( \prod_{i =1}^{k-1} W_{t_{i}} \right) E_{\theta^{\ast}} \left ( W_{t_{k}}  |  W_{t_{1}}, \ldots, W_{t_{k-1}} \right) \right]. \label{eq: expectation of V rewritten two-fold}
\end{align}
By construction of $W_{i}(\theta, \theta')$, we have $E_{\theta^{\ast}}(W_{i} | S_{i} = S) \leq 1 - c_{u}$, which follows from \eqref{eq: lower bound for EUgivenS}. Using this, the inner conditional expectation can be bounded as
\begin{align}
E_{\theta^{\ast}} \left ( W_{t_{k}}  |  W_{t_{1}}, \ldots, W_{t_{k-1}} \right)  &= \sum_{S \subset [K]} P_{\theta^{\ast}} (S_{t_{k}} = S | W_{t_{1}},\ldots, W_{t_{k-1}})  E_{\theta^{\ast}} (W_{t_{k}} | S_{t_{k}} = S) \nonumber \\
& \leq \sum_{S\subset [K]} P_{\theta^{\ast}} (S_{t_{k}} = S | W_{t_{1}},\ldots, W_{t_{k-1}}) \left( 1 - c_{u} \right) \nonumber  \\
& = \left( 1 - c_{u} \right). \label{eq: conditional expectation bounded}
\end{align}
Combining \eqref{eq: expectation of V rewritten two-fold} and \eqref{eq: conditional expectation bounded}, we have
\begin{equation} \label{eq: prodV recursion downwards}
E_{\theta^{\ast}}\left( \prod_{i =1}^{k} W_{t_{i}}  \right) \leq \left( 1 - c_{u} \right) E_{\theta^{\ast}} \left ( \prod_{i =1}^{k-1} W_{t_{i}} \right).
\end{equation}
By Lemmas \ref{lem: Upperbound for U} and \ref{lem: Lower bound for EU}, we have $V_{t}(\theta, \theta') \leq \| \theta - \theta'\|^{2}$ and $E_{\theta^{\ast}}[V_{t}(\theta, \theta')] \geq c_{u} \| \theta - \theta'\|^{2}$. Thus, we necessarily have $c_{u} < 1$. Therefore, the recursion in \eqref{eq: prodV recursion downwards} can be used until $k = 1$ to obtain the desired result.
\end{proof}
Note that $W_{t}(\theta, \theta')$ is bounded as $0 \leq W_{i}(\theta, \theta') \leq 1$. We now quote a critical theorem from \citet[Theorem 3.2]{Pelekis_and_Ramon_2017} regarding the sum of dependent and bounded random variables, which will be useful for bounding $\sum_{t = 1}^{n} V_{t}(\theta, \theta')$.

\begin{thm}{\citep[Theorem 3.2]{Pelekis_and_Ramon_2017}} \label{thm: Hoeffding dependent}
Let $W_{1}, \ldots, W_{n}$ be random variables, such that $0 \leq W_{t} \leq 1$, for $t = 1,\ldots, n$. Fix a real number $\tau \in (0, n)$ and let $k$ be any positive integer, such that $0 < k < \tau$. Then
\begin{align}
\mathbb{P}\left(\sum_{t = 1}^{n} W_{t} \geq \tau \right) \leq \frac{1}{ \binom{\tau}{k}} \sum_{A \subseteq \{1, \ldots, n \}: |A| = k} \mathbb{E}\left[ \prod_{i \in A} W_{i} \right],
\end{align}
where $\binom{\tau}{k} = \frac{\tau (\tau-1)\ldots (\tau - k + 1)}{k!}$.
\end{thm}

In the following, we apply Theorem \ref{thm: Hoeffding dependent} for $\sum_{t = 1}^{n} V_{t}(\theta, \theta')$. 
\begin{lem} \label{lem: almost concave}
For every $\theta, \theta' \in \Delta$ and $a \in (0, 1)$,
\[
\lim_{n \rightarrow \infty} P_{\theta^{\ast}}\left( \frac{1}{n}\sum_{t = 1}^{n} V_{t}(\theta, \theta') \leq a c_{u} \| \theta - \theta' \|^{2} \right) = 0.
\]
\end{lem}
\begin{proof}
For any integer $k < n (1 - a c_{u}) < n$, we have
\begin{align*}
P_{\theta^{\ast}}\left( \frac{1}{n}\sum_{t = 1}^{n} V_{t}(\theta, \theta') \leq a c_{u} \| \theta - \theta' \|^{2} \right) &=  
P_{\theta^{\ast}}\left( \| \theta - \theta' \|^{2} \frac{1}{n} \sum_{t = 1}^{n} (1 - W_{t}(\theta, \theta')) \leq a c_{u} \| \theta - \theta' \|^{2} \right) \\
& = P_{\theta^{\ast}}\left( \frac{1}{n} \sum_{t = 1}^{n} (1 - W_{t}(\theta, \theta')) \leq  a c_{u}  \right) \\
& = P_{\theta^{\ast}}\left( \sum_{t = 1}^{n} W_{t}(\theta, \theta') \geq n \left( 1 -  a c_{u} \right) \right) \\
& = P_{\theta^{\ast}}\left( \sum_{t = 1}^{n} W_{t}(\theta, \theta') \geq n \left( 1 -  a c_{u} \right) \right)  \\
& \leq \frac{1}{\binom{n \left( 1 -  a c_{u} \right)}{k}} \sum_{A \subseteq \{1, \ldots, n \}: |A| = k} E_{\theta^{\ast}} \left[ \prod_{i \in A} W_{i}(\theta, \theta') \right]  \\
& \leq \frac{1}{\binom{n \left( 1 -  a c_{u} \right)}{k}} \binom{n}{k} \left( 1 - c_{u} \right)^{k},
\end{align*}
where the last two lines follow from Theorem \ref{thm: Hoeffding dependent} and Lemma \ref{lem: bound on the expectation of product of V}, respectively. Select $k = k^{\ast} = \lceil n(1 - a) \rceil$ and note that when $n > 1/(a - ac_{u})$ one always has $k^{\ast} < n (1 - a c_{u})$. Then, for $n > 1/(a-a c_{u})$,
\begin{align}
P_{\theta^{\ast}}\left( \frac{1}{n}\sum_{t = 1}^{n} V_{t}(\theta, \theta') \leq a c_{u} \| \theta - \theta' \|^{2} \right) & \leq \frac{1}{ \binom{n \left( 1 -  a c_{u} \right)}{k^{\ast}}} { \binom{n}{k^{\ast}}} \left( 1 - c_{u} \right)^{k^{\ast}} \nonumber \\
& = \left(1 - c_{u} \right)^{k^{\ast}} \prod_{i = 1}^{k^{\ast}} \frac{n-i+1}{n (1 - a c_{u}) - i + 1}.\label{eq: upper bound for PU}
\end{align}
The right-hand side does not depend on $\theta, \theta'$ and converges to $0$.
\end{proof}

\paragraph{Smoothness of $\ln h_{S, \epsilon}(y | \theta)$:}
Next, we establish the $L$-smoothness of $h_{S, \epsilon}(y | \theta)$ as a function of $\theta$ for any $y \in [K]$ and $S \subset [K]$. Some technical lemmas are needed first.
\begin{lem} \label{lem: lame lemma 2}
For $x > 0$, $\ln(1 + x) \geq x - \frac{1}{2}x^{2}$.
\end{lem}
\begin{proof}
The function $\ln(1 + x) - x + 0.5 x^{2}$ is $0$ at $x = 0$ and its derivative $1/(1+x) - 1 + x = \frac{1}{1+x} - (1 - x) = \frac{1 - (1-x^{2})}{1+x} = x^{2}/(1+x) > 0$ when $x > 0$.
\end{proof}

\begin{lem}\label{lem: lame lemma 3}
For $x > 0$, $\ln(x + 1) \leq 1 - \frac{1}{x+1} + \frac{1}{2} x^{2}$.
\end{lem}

\begin{proof}
The function $\ln(x + 1) - 1 + \frac{1}{x+1} - \frac{1}{2} x^{2}$ is $0$ at $x = 0$ and its derivative,
\[
\frac{1}{x+1} - \frac{1}{(x+1)^{2}} - x = \frac{x - x (x+1)^{2}}{(x+1)^{2}} = \frac{x(1 -(x+1)^{2})}{(x+1)^{2}},
\]
is negative for $x > 0$.
\end{proof}

\begin{lem} \label{lem: smoothness of h}
There exists an $L_{0} > 0$ such that, for all $S \subset [K]$ and $y \in [K]$, the function $\ln h_{S, \epsilon}(y | \theta)$ is an $L_{0}$-smooth function of $\theta$. That is for all $\theta, \theta' \in \Delta$, we have
\[
\ln h_{S, \epsilon}(y | \theta) - \ln h_{S, \epsilon}(y | \theta') \leq \nabla_{\vartheta} \ln h_{S, \epsilon}(y | \theta)^{\top} (\vartheta - \vartheta') + L_{0} \| \theta - \theta' \|^{2}.
\]
\end{lem}

\begin{proof}
Assume $h_{S, \epsilon}(y| \theta) \geq h_{S, \epsilon}(y | \theta')$. Using Lemma \ref{lem: lame lemma 3} with $x = \frac{h_{S, \epsilon}(y| \theta)}{h_{S, \epsilon}(y| \theta')} -1 \geq 0$, we have
\begin{align}
\ln h_{S, \epsilon}(y| \theta) - \ln h_{S, \epsilon}(y| \theta') &\leq 1 - \frac{h_{S, \epsilon}(y| \theta')}{h_{S, \epsilon}(y| \theta)} + \frac{1}{2} \left(\frac{h_{S, \epsilon}(y| \theta)}{h_{S, \epsilon}(y| \theta')} -1 \right)^{2} \nonumber\\
& =\nabla_{\vartheta} \ln h_{S, \epsilon}(y | \theta)^{\top} (\vartheta - \vartheta') +  \frac{1}{2} \left(\frac{h_{S, \epsilon}(y| \theta) - h_{S, \epsilon}(y| \theta')}{h_{S, \epsilon}(y| \theta')} \right)^{2}  \nonumber \\
& \leq \nabla_{\vartheta} \ln h_{S, \epsilon}(y | \theta)^{\top} (\vartheta - \vartheta') +  \frac{1}{2 c^{2}} \left(h_{S, \epsilon}(y| \theta) - h_{S, \epsilon}(y| \theta') \right)^{2} \label{eq: same inequality for difference between lnh}  \\
& \leq \nabla_{\vartheta} \ln h_{S, \epsilon}(y | \theta)^{\top} (\vartheta - \vartheta') +  \frac{1}{2 c^{2}} \| \theta - \theta'\|^{2}, \nonumber 
\end{align}
where the first term in the second line follows from Lemma \ref{lem: inner product of gradient and difference}, the third line follows from Lemma \ref{lem: bounds for Gtheta}, and the last line follows from Lemma \ref{lem: Upperbound for U}.

Now, assume $h_{S, \epsilon}(y| \theta) \leq h_{S, \epsilon}(y | \theta')$. By Lemma \ref{lem: lame lemma 2} with $x = \frac{h_{S, \epsilon}(y| \theta')}{h_{S, \epsilon}(y| \theta)} -1 \geq 0$, 
\[
\ln h_{S, \epsilon}(y| \theta') - \ln h_{S, \epsilon}(y| \theta) \geq \left( \frac{h_{S, \epsilon}(y| \theta')}{h_{S, \epsilon}(y | \theta)} - 1\right) - \frac{1}{2}\left(  \frac{h_{S, \epsilon}(y| \theta')}{h_{S, \epsilon}(y | \theta)} - 1 \right)^{2},
\]
or, reversing the sign of inequality,
\begin{align*}
\ln h_{S, \epsilon}(y | \theta) - \ln h_{S, \epsilon}(y| \theta') &  \leq \left( 1- \frac{h_{S, \epsilon}(y| \theta')}{h_{S, \epsilon}(y | \theta)} \right) + \frac{1}{2}\left(  \frac{h_{S, \epsilon}(y| \theta')}{h_{S, \epsilon}(y | \theta)} - 1 \right)^{2} \\
& \leq \nabla_{\vartheta} \ln h_{S, \epsilon}(y | \theta)^{\top} (\vartheta - \vartheta') +  \frac{1}{2 c^{2}} \left(h_{S, \epsilon}(y| \theta) - h_{S, \epsilon}(y| \theta') \right)^{2}, \label{eq: same inequality for difference between lnh}
\end{align*}
where the third line follows from Lemma \ref{lem: bounds for Gtheta}. Hence, we have arrived at the same inequality as \eqref{eq: same inequality for difference between lnh}. Hence, Lemma \ref{lem: smoothness of h} holds with $L_{0} = \frac{1}{2 c^{2}}$.
\end{proof}

\paragraph{Second moment of the gradient at $\theta^{\ast}$:}
Let the average log-marginal likelihoods be defined as 
\begin{equation} \label{eq: F as average gradient of log marginal lkl}
 \Phi_{n}(\theta) := \frac{1}{n} \sum_{t = 1}^{n} \ln h_{S_{t}, \epsilon}(Y_{t} | \theta), \quad n \geq 1.
\end{equation}
The following bound on the second moment of this average at $\theta^{\ast}$ will be useful.
\begin{lem} \label{lem: expectation of F}
For $ \Phi_{n}(\theta)$ defined in \eqref{eq: F as average gradient of log marginal lkl}, we have
\begin{align*}
E_{\theta^{\ast}} \left[ \| \nabla_{\vartheta} \Phi_{n}(\theta^{\ast}) \|^{2} \right] \leq \frac{1}{n} \max_{S \subset [K]} \textup{Tr} \left[ F(\theta^{\ast}; S, \epsilon) \right].
\end{align*}
\end{lem}

\begin{proof}
First, we evaluate the mean at $\theta = \theta^{\ast}$.
\begin{align*}
E_{\theta^{\ast}} \left[ \nabla_{\vartheta} \Phi_{n}(\theta^{\ast}) \right] = \frac{1}{n} \sum_{t = 1}^{n} E_{\theta^{\ast}} \left[ \nabla_{\vartheta} \ln h_{S_{t}, \epsilon}(Y_{t}|\theta^{\ast})\right].
\end{align*}
Focusing on a single term, 
\begin{align*}
E_{\theta^{\ast}} \left[ \nabla_{\vartheta} \ln  h_{S_{t}, \epsilon}(Y_{t}|\theta^{\ast}) \right] &= \sum_{S \subset [K]} P_{\theta^{\ast}}(S_{t} = S) E_{\theta^{\ast}} \left[ \nabla_{\vartheta} \ln h_{S, \epsilon}(Y_{t} |\theta^{\ast}) | S_{t} = S \right].
\end{align*}
Each term in the sum is equal to $0$, since
\begin{equation} \label{eq: zero expectation of grad ln}
E_{\theta^{\ast}} \left[ \nabla_{\vartheta} \ln h_{S, \epsilon}(Y_{t}| \theta^{\ast}) | S_{t} = S \right]  = \sum_{k = 1}^{K} \nabla_{\vartheta} \ln h_{S, \epsilon}(k | \theta^{\ast}) h_{S, \epsilon}(k | \theta^{\ast}) = 0.
\end{equation}
For the second moment at $\theta = \theta^{\ast}$,
\begin{align*}
E_{\theta^{\ast}} \left[ \nabla_{\vartheta}\Phi_{n}(\theta^{\ast}) \nabla_{\vartheta}\Phi_{n}(\theta^{\ast})^{\top} \right] = & \frac{1}{n^{2}} \sum_{t = 1}^{n} E_{\theta^{\ast}} \left[ \nabla_{\vartheta}\ln h_{S_{t}, \epsilon}(Y_{t}| \theta^{\ast}) \nabla_{\vartheta}\ln h_{S_{t}, \epsilon}(Y_{t}|\theta^{\ast})^{\top} \right]\\
&+ \frac{2}{n^{2}} \sum_{t=1}^{n} \sum_{t' = 1}^{t-1} E_{\theta^{\ast}}\left[ \nabla_{\vartheta}\ln h_{S_{t}, \epsilon}(Y_{t}| \theta^{\ast}) \nabla_{\vartheta}\ln h_{S_{t'}}(Y_{t'}|\theta^{\ast})^{\top} \right].
\end{align*}
For the diagonal terms, for all $t = 1, \ldots, n$, we have
\begin{align*}
& E_{\theta^{\ast}} \left[ \nabla_{\vartheta}\ln h_{S_{t}, \epsilon}(Y_{t}| \theta^{\ast}) \nabla_{\vartheta}\ln h_{S_{t}, \epsilon}(Y_{t}| \theta^{\ast})^{\top} \right] \\
& \quad\quad= \sum_{S \subset [K]} P_{\theta^{\ast}}(S_{t}=S) E_{\theta^{\ast}} \left[ \nabla_{\vartheta}\ln h_{S, \epsilon}(Y_{t} | \theta^{\ast}) \nabla_{\vartheta}\ln h_{S, \epsilon}(Y_{t} | \theta^{\ast})^{\top} | S_{t} = S\right] \\
&\quad\quad= \sum_{S \subset [K]} P_{\theta^{\ast}}(S_{t}=S) F(\theta^{\ast}; S, \epsilon).
\end{align*}
For the cross terms, for $1 \leq t' < t \leq n$,
\begin{align*}
E_{\theta^{\ast}} \left[ \nabla_{\vartheta}\ln h_{S_{t}, \epsilon}(Y_{t}| \theta^{\ast}) \nabla_{\vartheta}\ln h_{S_{t'}}(Y_{t'}| \theta^{\ast})^{\top} \right] &= E_{\theta^{\ast}}\left\{E_{\theta^{\ast}}\left[ \nabla_{\vartheta}\ln h_{S_{t}, \epsilon}(Y_{t}| \theta^{\ast}) \nabla_{\vartheta}\ln h_{S_{t'}}(Y_{t'}| \theta^{\ast})^{\top} | Y_{t'}, S_{t'}\right] \right\} \\
&= E_{\theta^{\ast}}  \left\{ E_{\theta^{\ast}} \left[ \nabla_{\vartheta}\ln h_{S_{t}, \epsilon}(Y_{t}| \theta^{\ast})  | Y_{t'}, S_{t'}\right] \nabla_{\vartheta}\ln h_{S_{t}, \epsilon}(Y_{t'}| \theta^{\ast})^{\top}  \right\}.
\end{align*}
The conditional expectation inside is zero, since, by \eqref{eq: zero expectation of grad ln}, 
\begin{align*}
E_{\theta^{\ast}}\left[ \nabla_{\vartheta}\ln h_{S_{t}, \epsilon}(Y_{t}|  \theta^{\ast})  | Y_{t'}, A_{t'}\right] = \sum_{S \subset [K]} P_{\theta^{\ast}}(S_{t} = S | Y_{t'}, S_{t'}) E_{\theta^{\ast}} \left[ \nabla_{\vartheta} \ln h_{S, \epsilon}(Y_{t} | \theta^{\ast} ) | S_{t} = S\right] = 0.
\end{align*}
Therefore, all the cross terms are zero,
\[
E_{\theta^{\ast}}\left[ \nabla_{\vartheta}\ln h_{S_{t}, \epsilon}(Y_{t}| \theta^{\ast}) \nabla_{\vartheta}\ln h_{S_{t'}}(Y_{t'}| \theta^{\ast})^{\top} \right] = 0,
\]
and hence, we arrive at
\begin{align}
E_{\theta^{\ast}} \left[ \nabla_{\vartheta}\Phi_{n}(\theta^{\ast}) \nabla_{\vartheta}\Phi_{n}(\theta^{\ast})^{\top} \right] = \frac{1}{n^{2}}  \sum_{t = 1}^{n} \sum_{S \subset [K]} P_{\theta^{\ast}}(S_{t}=S) \left[ F(\theta^{\ast}; S, \epsilon) \right]
\end{align}
for the second moment of the gradient of $\Phi_{n}(\theta^{\ast})$. Therefore,
\begin{align*}
E_{\theta^{\ast}} \left[ \| \nabla_{\vartheta}\Phi_{n}(\theta^{\ast}) \|^{2} \right]  &= E_{\theta^{\ast}} \left[ \textup{Tr} \left( \nabla_{\vartheta}\Phi_{n}(\theta^{\ast}) \nabla_{\vartheta}\Phi_{n}(\theta^{\ast})^{\top} \right) \right] \\
& = \textup{Tr} \left( E_{\theta^{\ast}} \left[ \nabla_{\vartheta}\Phi_{n}(\theta^{\ast}) \nabla_{\vartheta}\Phi_{n}(\theta^{\ast})^{\top} \right] \right) \ \\
& = \frac{1}{n^{2}}  \sum_{t = 1}^{n} \sum_{S \subset [K]} P_{\theta^{\ast}}(S_{t}=S) \textup{Tr} (F(\theta^{\ast}; S, \epsilon)) \\
&\leq \frac{1}{n} \max_{S \subset [K]} \textup{Tr} \left( F(\theta^{\ast}; S, \epsilon) \right),
\end{align*}
which concludes the proof.
\end{proof}

\subsubsection{Convergence of the posterior distribution} \label{appndx: Convergence of the posterior distribution}
Let $\mu \in (0, 1)$ and, for $\theta, \theta' \in \Delta$, define
\begin{equation} \label{eq: mathcalE}
\mathcal{E}^{\mu}_{n}(\theta, \theta') := \mu c_{u} \| \theta - \theta'\|^{2} -  \frac{1}{n} \sum_{t = 1}^{n} V_{t}(\theta, \theta'),
\end{equation}
where $V_{t}(\theta, \theta')$ was defined in \eqref{eq: Ut definition} and $c_{u} > 0$ was defined in the proof of Lemma \ref{lem: Lower bound for EU}, respectively. The proof of Theorem \ref{thm: convergence to posterior} requires the following lemma concerning $\mathcal{E}^{\mu}_{n}(\theta, \theta')$.
\begin{lem} \label{lem: convergence of integral of expE}
There exists $\mu \in (0, 1)$ such that, for any $\varepsilon > 0$, we have
\[
\lim_{n \rightarrow \infty} P_{\theta^{\ast}} \left(\int_{\Delta} e^{ n m_{0} \mathcal{E}^{\mu}_{n}(\theta, \theta^{\ast}) } \mathrm{d}\theta > e^{\varepsilon} \right)  = 0.
\]
\end{lem}

\begin{proof}
Define the product measure
\begin{equation} \label{eq: product measure}
P^{\otimes}(\mathrm{d}(\theta, \cdot)) := \frac{\mathrm{d}\theta}{|\Delta|} \times d P_{\theta^{\ast}}(\cdot) 
\end{equation}
for random variables $( \Theta \in \Delta, \{ S_{t} \subset \{ 1,\ldots, K\}, Y_{t} \in [K] \}_{t \geq 1})$, where $\mathrm{d}\theta$ is the Lebesgue measure for $\vartheta$ restricted to $\Delta$ and $|\Delta| := \int_{\Delta} \mathrm{d}\theta$. We will show that the parameter $\mu$ in \eqref{eq: mathcalE} can be chosen such that the collection of random variables
\[
\mathcal{C} := \{f_{n}:= \max\{ 1, e^{ n m_{0} \mathcal{E}^{\mu}_{n}(\Theta, \theta^{\ast}) } \}: n \geq  1\}
\]
is uniformly integrable with respect to $P^{\otimes}$. For uniform integrability, we need to show that for any $\varepsilon > 0$, there exists a $K > 0$ such that
\[
E^{\otimes} [ |f_{n} | \cdot \mathbb{I}( f_{n}  > K)] < \varepsilon, \quad \forall n \geq 1.
\]
For any $K > 1$ and $n \geq 1$, we have 
\begin{align*}
E^{\otimes} [ |f_{n} | \cdot \mathbb{I}( f_{n}  > K)] &= E^{\otimes} [ e^{ n m_{0} \mathcal{E}^{\mu}_{n}(\Theta, \theta^{\ast}) } \mathbb{I}( f_{n}  > K)] \\
&\leq \sup_{\theta \in \Delta}  e^{ n m_{0} \mathcal{E}^{\mu}_{n}(\theta, \theta^{\ast}) } P^{\otimes} (f_{n} > K)  \\
&= \sup_{\theta \in \Delta} e^{ n m_{0} \mathcal{E}^{\mu}_{n}(\theta, \theta^{\ast}) } \int_{\Delta } P_{\theta^{\ast}} \left(\mathcal{E}^{\mu}_{n} (\theta, \theta^{\ast}) >\frac{\ln K}{n m_{0}} \right)  \frac{\mathrm{d}\theta}{|\Delta |} \\
& \leq e^{ n \mu m_{0}  c_{u}} P_{\theta^{\ast}} \left(\mathcal{E}^{\mu}_{n} (\theta, \theta^{\ast}) > 0 \right),
\end{align*}
where the last line follows from $\mathcal{E}^{\mu}_{n} (\theta, \theta^{\ast}) \leq \mu c_{u} \| \theta - \theta^{\ast}\|^{2} \leq  \mu c_{u}$. Using \eqref{eq: upper bound for PU}, the last expression can be upper-bounded as
\begin{align}
e^{ n m_{0} \mu c_{u} } P_{\theta^{\ast}} \left(\mathcal{E}^{\mu}_{n} (\theta, \theta^{\ast}) >0 \right) &= e^{ n \mu m_{0}  c_{u}} P_{\theta^{\ast}} \left( \frac{1}{n} \sum_{t = 1}^{t} V_{t} (\theta, \theta^{\ast}) < \mu c_{u} \| \theta - \theta'\|^{2} \right) \nonumber \\
& \leq e^{ n \mu m_{0}  c_{u}}  \left(1 - c_{u} \right)^{\lceil n(1 - \mu) \rceil} \prod_{i = 1}^{\lceil n(1 - \mu) \rceil} \frac{n-i+1}{n (1 - \mu c_{u}) - i + 1}.\label{eq: bound for uniform integrable}
\end{align}
The parameter $\mu$ can be arranged such that \eqref{eq: bound for uniform integrable} converges to $0$. For such $\mu$, we have that for any $\varepsilon$ there exists a $N_{\varepsilon} > 0$ such that for all $n > N_{\varepsilon}$, $E^{\otimes} [ |f_{n} | \cdot \mathbb{I}( f_{n}  > K)] < \varepsilon$ for any $K > 0$. Finally, choose $K_{\varepsilon} = e^{N_{\varepsilon} m_{0} \mu c_{u}}$ so that $E^{\otimes} [ | f_{n} | \cdot \mathbb{I}( f_{n}  > K_{\varepsilon})] < \varepsilon$ for any $n \geq 1$. Hence, $\mathcal{C}$ is uniformly integrable for a suitable choice of $\mu$.

Next, we show that each $f_{n}$ in $\mathcal{C}$ converges in probability to $1$. The convergence is implied by the fact that for every $\theta \in \Delta$ and $\varepsilon > 0$, we have
\[
P_{\theta^{\ast}}(\max\{1, e^{n m_{0} \mathcal{E}^{\mu}_{n}(\theta, \theta^{\ast})} \} > e^{\varepsilon}) = P_{\theta^{\ast}}(\mathcal{E}^{\mu}_{n}(\theta, \theta^{\ast}) > \varepsilon) \rightarrow 0.
\]
by Lemma \ref{lem: almost concave}. Since $\mathcal{C}$ is uniformly integrable, the Vitali convergence theorem ensures that $f_{n}$ converges in distribution (with respect to $P^{\otimes}$) to $1$, i.e., $\lim_{n \rightarrow \infty} E^{\otimes} (f_{n}) = 1$. Since $P^{\otimes} = \frac{\mathrm{d}\theta}{|\Delta|} \times d P_{\theta^{\ast}}(\cdot)$ is a product measure as defined in \eqref{eq: product measure}, the stated limit implies that
\[
E_{\theta^{\ast}} \left[ \int_{\Delta} \max \{ 1, e^{ n m_{0} \mathcal{E}^{\mu}_{n}(\theta, \theta^{\ast}) } \} \mathrm{d}\theta \right]\rightarrow 1, 
\] 
that is, the sequence $\int_{\Delta} \max \{ 1, e^{ n m_{0} \mathcal{E}^{\mu}_{n}(\theta, \theta^{\ast}) } \} \mathrm{d}\theta $ converges to $1$ in distribution with respect to $P_{\theta^{\ast}}$. Since convergence in distribution to a constant implies convergence in probability, we have  
\begin{equation} \label{eq: convergence of max of 1 and int}
\int_{\Delta} \max \{ 1, e^{ n m_{0} \mathcal{E}^{\mu}_{n}(\theta, \theta^{\ast}) } \} \mathrm{d}\theta  \overset{P_{\theta^{\ast}}}{\rightarrow} 1.
\end{equation}
Finally, since we have
\[
\int_{\Delta} e^{ n m_{0} \mathcal{E}^{\mu}_{n}(\theta, \theta^{\ast}) }\mathrm{d}\theta \leq \int_{\Delta} \max \{1, e^{ n m_{0} \mathcal{E}^{\mu}_{n}(\theta, \theta^{\ast}) } \} \mathrm{d}\theta,
\]
and the right-hand side converges in probability to 1, we conclude.
\end{proof}

\begin{proof}[Proof of Theorem \ref{thm: convergence to posterior}]
Writing down Lemma \ref{lem: concavity} with $\theta^{\ast}$ and any $\theta \in \Delta$ separately for $t = 1, \ldots, n$, summing the inequalities and dividing by $n$, we obtain
\begin{align*}
\Phi_{n}(\theta^{\ast}) - \Phi_{n}(\theta) & \geq \nabla_{\vartheta}\Phi_{n}(\theta^{\ast})^{\top} (\vartheta^{\ast} - \vartheta) + m_{0} \sum_{t = 1}^{n}  V_{t}(\theta, \theta^{\ast}) \\
&= \nabla_{\vartheta}\Phi_{n}(\theta^{\ast})^{\top} (\vartheta^{\ast} - \vartheta)  + m  \| \theta - \theta^{\ast}\|^{2} - m_{0}  \mathcal{E}^{\mu}_{n}(\theta, \theta^{\ast}).
\end{align*}
where $m := \mu m_{0} c_{u}$. Reversing the sign, 
\begin{align*}
\Phi_{n}(\theta) - \Phi_{n}(\theta^{\ast}) \leq \nabla_{\vartheta}\Phi_{n}(\theta^{\ast})^{\top} (\vartheta - \vartheta^{\ast}) - m \|\theta^{\ast} - \theta \|^{2} + m_{0} \mathcal{E}^{\mu}_{n}(\theta, \theta^{\ast}).
\end{align*}
Using Cauchy-Schwarz inequality for the first term on the right-hand side, we get
\begin{align*}
\Phi_{n}(\theta) - \Phi_{n}(\theta^{\ast})  \leq \| \nabla_{\vartheta}\Phi_{n}(\theta^{\ast}) \| \|\theta - \theta^{\ast}\| - m \|\theta^{\ast} - \theta \|^{2} + m_{0} \mathcal{E}^{\mu}_{n}(\theta, \theta^{\ast}).
\end{align*}
Using Young's inequality $u v \leq \frac{u^{2} }{2 \kappa} + \frac{v^{2} \kappa}{2 }$ for the second term with $u = \|\nabla_{\vartheta}\Phi_{n}(\theta^{\ast}) \|$, $v = \| \theta^{\ast} - \theta \|$, and $\kappa = m$, we get
\begin{equation} \label{eq: Difference between F - 1}
\Phi_{n}(\theta) - \Phi_{n}(\theta^{\ast})  \leq \frac{\| \nabla_{\vartheta}\Phi_{n}(\theta^{\ast}) \|^{2}}{2 m} - \frac{m}{2}  \| \theta^{\ast} - \theta \|^{2} + m_{0} \mathcal{E}^{\mu}_{n}(\theta, \theta^{\ast}).
\end{equation}
Similarly, using Lemma \ref{lem: smoothness of h} with $\theta^{\ast}$ and any $\theta' \in \Delta$ for $t = 1, \ldots, n$, summing the inequalities and dividing by $n$, we obtain
\[
\Phi_{n}(\theta^{\ast}) - \Phi_{n}(\theta') \leq \nabla_{\vartheta}\Phi_{n}(\theta^{\ast})^{\top} (\vartheta^{\ast} - \vartheta') + L_{0} \|\theta^{\ast} - \theta' \|^{2}.
\]
Again, using Cauchy-Schwarz inequality and Young's inequality $u v \leq \frac{u^{2} }{2 \kappa} + \frac{v^{2} \kappa}{2 }$ with $u = \| \nabla_{\vartheta}\Phi_{n}(\theta^{\ast}) \|$, $v = \| \theta^{\ast} - \theta \|$, and $\kappa = 2 L_{0}$, we get
\begin{align}
\Phi_{n}(\theta^{\ast}) - \Phi_{n}(\theta') & \leq \| \nabla_{\vartheta}\Phi_{n}(\theta^{\ast}) \| \|\theta^{\ast} - \theta' \| + L_{0} \|\theta^{\ast} - \theta' \|^{2} \nonumber  \\
 & \leq \frac{\| \nabla_{\vartheta}\Phi_{n}(\theta^{\ast}) \|^{2}}{4 L_{0}} + 2L_{0} \| \theta^{\ast} - \theta' \|^{2} \nonumber \\
 & = \frac{\| \nabla_{\vartheta}\Phi_{n}(\theta^{\ast}) \|^{2}}{2 L} + L \| \theta^{\ast} - \theta' \|^{2}, \label{eq: Difference between F - 2}
\end{align}
where we let $L := 2 L_{0}$. Summing the inequalities in \eqref{eq: Difference between F - 1} and \eqref{eq: Difference between F - 2}, we obtain
\begin{equation} \label{eq: Difference between F - 3}
\Phi_{n}(\theta)  - \Phi_{n}(\theta') \leq \left( \frac{1}{2L} + \frac{1}{2m} \right) \| \nabla_{\vartheta}\Phi_{n}(\theta^{\ast}) \|^{2} - \frac{m}{2}  \| \theta^{\ast} - \theta \|^{2} + \frac{L}{2}  \| \theta^{\ast} - \theta' \|^{2} + m_{0} \mathcal{E}^{\mu}_{n}(\theta, \theta^{\ast}).
\end{equation}
Let $a \in (0, 1)$ be a constant and define the sequences 
\begin{align*}
\Omega_{n} &:= \{ \theta \in \Delta: \|\theta - \theta^{\ast}\|^{2} < \max\{ 4/m, 4/L\} n^{-a} \}, \quad n\geq 1. \\
A_{n} &:=  \{ \theta \in \Delta: \|\theta - \theta^{\ast}\|^{2} > \max\{ 4/m, 4/L\} n^{-a} \}, \quad n\geq 1. \\
B_{n} &:= \{ \theta \in \Delta: \|\theta - \theta^{\ast}\|^{2} \leq \min\{ 2/m, 2/L\} n^{-a} \}, \quad n\geq 1.
\end{align*}
For $\theta \in A_{n}$ and $\theta' \in B_{n}$, \eqref{eq: Difference between F - 3} can be used to obtain
\begin{equation*}
\Phi_{n}(\theta)  - \Phi_{n}(\theta') \leq \left( \frac{1}{2L} + \frac{1}{2m} \right) \| \nabla_{\vartheta}\Phi_{n}(\theta^{\ast}) \|^{2} - n^{-a}\left( \max \left\{ 2, \frac{2 m}{L} \right\} - \min\left\{ \frac{L}{m}, 1 \right\} \right) + m_{0} \mathcal{E}^{\mu}_{n}(\theta, \theta^{\ast}).
\end{equation*}
Noting that $\max \left\{ 2, \frac{2 m}{L} \right\} - \min\left\{ \frac{L}{m}, 1\right\} \geq 1$, we have 
\begin{equation} \label{eq: Difference between F - 4}
\Phi_{n}(\theta)  - \Phi_{n}(\theta') \leq \left( \frac{1}{2L} + \frac{1}{2m} \right) \| \nabla_{\vartheta}\Phi_{n}(\theta^{\ast}) \|^{2} - n^{-a} + m_{0} \mathcal{E}^{\mu}_{n}(\theta, \theta^{\ast}), \quad \theta \in A_{n}; \theta' \in B_{n}.
\end{equation}
Multiplying \eqref{eq: Difference between F - 4} with $n$,  exponentiating, and multiplying the ratio of the priors, we get
\begin{align}
 \frac{\eta(\theta)\exp\{ n \Phi_{n}(\theta) \}}{\eta(\theta')\exp\{ n \Phi_{n}(\theta') \}} & \leq \frac{\eta(\theta)}{\eta(\theta')}\exp \left[ \left( \frac{1}{2L} + \frac{1}{2m} \right)  n \| \nabla_{\vartheta}\Phi_{n}(\theta^{\ast}) \|^{2} - n^{1-a} + n m_{0} \mathcal{E}^{\mu}_{n}(\theta, \theta^{\ast}) \right] \nonumber \\
&  \leq C_{\eta, n} \exp \left[ C_{1} n \| \nabla_{\vartheta}\Phi_{n}(\theta^{\ast}) \|^{2} - n^{1-a} + n m_{0} \mathcal{E}^{\mu}_{n}(\theta, \theta^{\ast}) \right]  \label{eq: Difference between F - 5}
\end{align}
for all $\theta \in A_{n}$ and $\theta' \in B_{n}$, where $C_{1} := \frac{1}{2L} + \frac{1}{2m}$ and $C_{\eta, n} := \sup_{\theta \in A_{n}, \theta' \in B_{n}} \frac{\eta(\theta)}{\eta(\theta')} $. The bound in \eqref{eq: Difference between F - 5} can be used to bound the ratio between the posterior probabilities $\Pi(A_{n} | Y_{1:n}, S_{1:n})$ and $\Pi(B_{n} | Y_{1:n}, S_{1:n})$, since
\begin{align*}
\frac{ \Pi(A_{n} | Y_{1:n}, S_{1:n}) }{ \Pi(B_{n} | Y_{1:n}, S_{1:n}) } &= \frac{\int_{A_{n}} \eta(\theta) \exp\{n \Phi_{n}(\theta) \}  \mathrm{d}\theta}{\int_{B_{n}} \eta(\theta) \exp\{n \Phi_{n}(\theta) \} \mathrm{d}\theta} \\
&= \frac{\int_{A_{n}} \frac{\eta(\theta) \exp\{n \Phi_{n}(\theta) \} }{\inf_{\theta' \in B_{n}} \eta(\theta') \exp\{n \Phi_{n}(\theta') \}} \mathrm{d}\theta}{\int_{B_{n}} \frac{\eta(\theta) \exp\{n \Phi_{n}(\theta) \}}{\inf_{\theta' \in B_{n}} \eta(\theta') \exp\{n \Phi_{n}(\theta') \}} \mathrm{d}\theta} \\
& \leq \frac{\int_{A_{n}} \frac{\eta(\theta) \exp\{n \Phi_{n}(\theta) \} }{\inf_{\theta' \in B_{n}} \eta(\theta') \exp\{n \Phi_{n}(\theta') \}} \mathrm{d}\theta}{\int_{B_{n}} \frac{\eta(\theta) \exp\{n \Phi_{n}(\theta) \}}{\inf_{\theta' \in B_{n}} \eta(\theta') \exp\{n \Phi_{n}(\theta') \}} \mathrm{d}\theta} \\
& \leq \frac{\int_{A_{n}}  C_{\eta, n} \exp \left[ C_{1} n \| \nabla_{\vartheta}\Phi_{n}(\theta^{\ast}) \|^{2} - n^{1-a} + n m_{0} \mathcal{E}^{\mu}_{n}(\theta, \theta^{\ast}) \right]  \mathrm{d}\theta}{\int_{B_{n}} 1 \mathrm{d}\theta} \\
& = \frac{1}{\text{Vol}(B_{n})}  C_{\eta, n} \exp \left[ C_{1}  n \| \nabla_{\vartheta}\Phi_{n}(\theta^{\ast}) \|^{2}  - n^{1-a} \right] \int_{A_{n}} \exp\left[ n m_{0} \mathcal{E}^{\mu}_{n}(\theta, \theta^{\ast}) \right] \mathrm{d}\theta,
\end{align*}
where $\text{Vol}(B_{n}) := \int_{B_{n} \cap \Delta} \mathrm{d}\theta$. Note that $B_{n}$ shrinks with $n$, so there exists a $N_{B}$ such that for $n > N_{B}$, the volume $B_{n}$ can be lower-bounded as
\[
B_{n} \geq \frac{1}{2(K-2)!} \frac{ (\sqrt{\pi} \min\{ 2/m, 2/L\} n^{-a})^{(K-1)}}{ \Gamma((K-1)/2+1)},
\]
where the factor $\frac{1}{2(K-2)!}$ corresponds to the worst-case situation where $\theta^{\ast}$ is on one of the corners of $\Delta$, such as $\theta^{\ast} = (1, 0, \ldots, 0)^{\top}$, and the rest is the volume of a $K-1$ dimensional sphere with radius $\min\{ 2/m, 2/L\} n^{-\alpha}$. The lower bound is the volume of the intersection of a simplex with a sphere centered at one of the sharpest corners of the simplex. Therefore, for $n > N_{B}$, ratio can further be bounded as 
\begin{equation*}
\frac{ \Pi(A_{n} | Y_{1:n}, S_{1:n}) }{ \Pi(B_{n} | Y_{1:n}, S_{1:n}) } \leq C_{2} C_{\eta, n}  \exp \left[ C_{1}  n \| \nabla_{\vartheta}\Phi_{n}(\theta^{\ast}) \|^{2} + (K-1) a \ln n - n^{1-a} \right]  \int_{A_{n}} \exp\left[ n m_{0} \mathcal{E}^{\mu}_{n}(\theta, \theta^{\ast})\right] \mathrm{d}\theta,
\end{equation*}
where $C_{2} := \frac{\Gamma((K-1)/2+1)  }{\min\{ 2/m, 2/L\}^{K-1} \pi^{(K-1)/2}}$ does not depend on $n$.

Next, we prove that the sequence of random variables
\[
Z_{n} := C_{\eta, n}  \exp \left[ C_{1} n \| \nabla_{\vartheta}\Phi_{n}(\theta^{\ast}) \|^{2} + (K-1) a \ln n - n^{1-a} \right]  \int_{A_{n}} \exp\left[ n m_{0} \mathcal{E}^{\mu}_{n}(\theta, \theta^{\ast})\right] \mathrm{d}\theta 
\]
converges to $0$ in probability, which in turn proves the convergence of $\frac{ \Pi(A_{n} | Y_{1:n}, S_{1:n}) }{ \Pi(B_{n} | Y_{1:n}, S_{1:n}) }$ in probability to $0$. To do that, we need to prove that for each $\varepsilon > 0$ and $\delta > 0$, there exists a $N > 0$ such that for all $n > N$ we have $P_{\theta^{\ast}}(Z_{n} \geq 2\varepsilon) < 2 \delta$. Fix $\varepsilon > 0$ and $\delta > 0$. 
\begin{itemize}
\item Firstly, by Assumption \ref{asmp: prior has mass around theta}, because $B_{n}$ shrinks towards $\theta^{\ast}$, there exists $N_{\eta} > 0$ such that  $C_{\eta, n} < B$ for all $n > N_{\eta}$. 
\item Next, let $\beta := C_{1} \max_{S \subset [K]} \textup{Tr} (F(\theta^{\ast}; S, \epsilon))/\delta$. Using Markov's inequality for $\| \nabla_{\vartheta}\Phi_{n}(\theta^{\ast}) \|^{2}$ with Lemma \ref{lem: expectation of F}, we have
\begin{align*}
P_{\theta^{\ast}} \left( \| \nabla_{\vartheta}\Phi_{n}(\theta^{\ast}) \|^{2} \geq \frac{1}{n}\frac{\beta}{C_{1}} \right) &\leq \frac{1}{n} \max_{S \subset \{1, \ldots,  K \}} \textup{Tr} (F(\theta; S, \epsilon)) \frac{ C_{1} n }{\beta} = \delta.
\end{align*}
Also, since the $n^{1 - a}$ dominates the term $\ln n$, one can choose an integer $N_{\Phi} > 0$ such that 
\[
\beta \leq \ln (\varepsilon/B) + n^{1 - a} - (K-1) a \ln n, \quad \forall n \geq N_{\Phi}.
\] 
\item Now we deal with the integral in $Z_{n}$. We have 
\[
P_{\theta^{\ast}} \left( \int_{A_{n}} \exp\left[ n m_{0} \mathcal{E}^{\mu}_{n}(\theta, \theta^{\ast})\right] \mathrm{d}\theta \geq 2 \right) \leq P_{\theta^{\ast}} \left( \int_{\Delta} \exp\left[ n m_{0} \mathcal{E}^{\mu}_{n}(\theta, \theta^{\ast})\right] \mathrm{d}\theta \geq 2 \right) \rightarrow 1,
\]
where the convergence is due to Lemma \ref{lem: convergence of integral of expE}. Hence, there exists a $N_{\mathcal{E}}$ such that for all $n > N_{\mathcal{E}}$,
\[
P_{\theta^{\ast}} \left( \int_{A_{n}} \exp\left[ n m_{0} \mathcal{E}^{\mu}_{n}(\theta, \theta^{\ast})\right] \mathrm{d}\theta \geq 2 \right) \leq \delta.
\]
\end{itemize}
Gathering the results, for $n > \max\{N_{\eta}, N_{\Phi}, N_{\mathcal{E}} \}$, we have
\begin{align*}
P_{\theta^{\ast}}(Z_{n} \geq \varepsilon) &\leq P_{\theta^{\ast}}\left( e^ {C_{1}  n \| \nabla_{\vartheta}\Phi_{n}(\theta^{\ast}) \|^{2} + (K-1) a \ln n - n^{1-a} } \geq \varepsilon/B \right) + P_{\theta^{\ast}} \left( \int_{\Delta} \exp\left[ n m_{0} \mathcal{E}^{\mu}_{n}(\theta, \theta^{\ast})\right] \mathrm{d}\theta \geq 2 \right) \\
& \leq P_{\theta^{\ast}} \left( C_{1}  n \| \nabla_{\vartheta}\Phi_{n}(\theta^{\ast}) \|^{2} + (K-1)a \ln n - n^{1-a}  \geq \ln (\varepsilon/B) \right) + \delta\\
& = P_{\theta^{\ast}} \left(  n \| \nabla_{\vartheta}\Phi_{n}(\theta^{\ast}) \|^{2}   \geq \frac{\ln (\varepsilon/B)  + n^{1-a} - (K-1) a \ln n}{C_{1} } \right) + \delta \\
& \leq P_{\theta^{\ast}} \left(  \| \nabla_{\vartheta}\Phi_{n}(\theta^{\ast}) \|^{2} \geq \frac{\beta}{C_{1}  n} \right) + \delta \\
& \leq 2\delta.
\end{align*}
(In the first line, we have used $\mathbb{P}(XY > pq) = 1 - \mathbb{P}(XY < pq) \leq 1 - \mathbb{P}(X < p \text{ and } Y < q) = \mathbb{P}(X > p \text{ or } Y > q) \leq \mathbb{P}(X > p) + \mathbb{P}(Y > q)$ for non-negative random variables $X, Y$ and positive $p, q$.) Therefore we have proved that $Z_{n} \rightarrow 0$ in probability. Finally, since $B_{n} \subset \Omega_{n}$, we have
\[
\frac{ \Pi(A_{n} | Y_{1:n}, S_{1:n}) }{ \Pi(\Omega_{n} | Y_{1:n}, S_{1:n}) } 
 \leq \frac{ \Pi(A_{n} | Y_{1:n}, S_{1:n}) }{ \Pi(B_{n} | Y_{1:n}, S_{1:n}) }  \leq  C_{2} Z_{n} 
\]
for all $n > N_{B}$. This implies that, 
\[
\frac{ \Pi(A_{n} | Y_{1:n}, S_{1:n}) }{ \Pi(\Omega_{n} | Y_{1:n}, S_{1:n}) } \overset{P_{\theta^{\ast}}}{\rightarrow} 0.
\]
Since $A_{n} = \Delta / \Omega_{n}$, as a result we get $\Pi(\Omega_{n} | Y_{1:n}, S_{1:n}) \overset{P_{\theta^{\ast}}}{\rightarrow} 1$. This concludes the proof.
\end{proof}

\subsubsection{Convergence of the expected frequency} \label{sec: Convergence of the expected frequency}
\begin{proof}[Proof of Theorem \ref{thm: choosing the best subset}]
Assumption \ref{asmp: unique maximizer} ensures that there exists a $\kappa_{0} > 0$ and $k^{\ast} \in \{0, \ldots, K-1\}$ such that for all $0 \leq k \neq k^{\ast} < K$,
\[
U(\theta^{\ast}; S^{\ast}, \epsilon) - U(\theta^{\ast}; \{\sigma_{\theta^{\ast}}(1), \ldots, \sigma_{\theta^{\ast}}(k)\}, \epsilon) \geq \kappa_{0}.
\]
By Assumption \ref{asmp: continuity of U}, there exists a $\delta_{1} > 0$ such that 
\[
\| \theta - \theta' \| \leq \delta_{1} \Rightarrow | U(\theta, S, \epsilon) - U(\theta', S, \epsilon) | < \kappa_{0}/2.
\]
Moreover, since the components of $\theta^{\ast}$ are strictly ordered, 
\[
\delta_{2} := \min_{k = 1, \ldots, K-1 }  ( \theta^{\ast}(k) - \theta^{\ast}(k+1) ) > 0.
\]
Choose $\delta = \min \{ \delta_{1}, \delta_{2}/\sqrt{2} \}$. Define the set
\[
\Omega_{\delta} = \{\theta \in \Delta: \|\theta- \theta^{\ast} \|^{2} \leq \delta^{2} \}.
\]
Then, for any $\theta \in \Omega_{\delta}$, $\sigma_{\theta} = \sigma_{\theta^{\ast}}$ and $S^{\ast}_{\theta} = S^{\ast}$. This implies that $\{\theta_{n} \in \Omega_{\delta}\} \subseteq \{ S_{n+1} = S^{\ast} \}$. Since perfect sampling is assumed, we have $Q(\mathrm{d}\theta_{t} | Y_{1:t}, S_{1:t}) = \Pi(\mathrm{d}\theta_{t} | Y_{1:t}, S_{1:t})$. Hence, 
\begin{equation} \label{eq: S is Sstar prob}
P_{\theta^{\ast}}(S_{n+1} = S^{\ast}) \geq E_{\theta^{\ast}} \left[ P_{\theta^{\ast}}(\theta_{n} \in \Omega_{\delta} | S_{1:n}, Y_{1:n}) \right]= E_{\theta^{\ast}} \left[ \Pi(\Omega_{\delta} | Y_{1:n}, S_{1:n}) \right]
\end{equation}
Recall the sequence of sets
\[
\Omega_{n} = \{ \theta \in \Delta: \|\theta - \theta^{\ast}\|^{2} \leq c n^{-a} \}
\]
defined in Theorem \ref{thm: convergence to posterior}. There exists an $N_{1} > 0$ such that $n > N_{1}$ we have $\Omega_{n} \subseteq \Omega_{\delta}$. For such $N_{1}$, we have 
\[
E_{\theta^{\ast}}\left[\Pi(\Omega_{\delta} | Y_{1:n}, S_{1:n}) \right] \geq E_{\theta^{\ast}}\left[ \Pi(\Omega_{n} | Y_{1:n}, S_{1:n}) \right], \quad n > N_{1}.
\]
Combining with \eqref{eq: S is Sstar prob}, we can write as
\begin{equation} \label{eq: convergence of PS LHS RHS}
P_{\theta^{\ast}}(S_{n+1} = S^{\ast}) \geq E_{\theta^{\ast}}\left[ \Pi(\Omega_{n} | Y_{1:n}, S_{1:n}) \right], \quad n > N_{1}.
\end{equation}
We will show that the right-hand side converges to $1$. To do that, fix $\varepsilon > 0$. By Theorem \ref{thm: convergence to posterior}, there exists a $N_{2} > 0$ such that
\[
P_{\theta^{\ast}} \left(\Pi(\Omega_{n} | Y_{1:n}, S_{1:n}) > \sqrt{1-\varepsilon} \right) > \sqrt{1 - \varepsilon}.
\]
This implies that, for $n > N_{2}$,
\[
E_{\theta^{\ast}} (\Pi(\Omega_{n} | Y_{1:n}, S_{1:n})) > \sqrt{1 - \varepsilon} \sqrt{1 - \varepsilon} + 0  (1 - \sqrt{1 - \varepsilon}) = 1 - \varepsilon.
\]
This shows that $E_{\theta^{\ast}}\left[ \Pi(\Omega_{n} | Y_{1:n}, S_{1:n}) \right] \rightarrow 1$ as $n \rightarrow \infty$.  
Since the right-hand side of \eqref{eq: convergence of PS LHS RHS} converges to $1$, so does the left-hand side. Therefore, we have proven \eqref{eq: convergence of probability}.

To prove \eqref{eq: sublinear regret}, we utilize the convergence of Cesaro means and write 
\[
\lim_{n \rightarrow \infty} \frac{1}{n} \sum_{t = 1}^{n} P_{\theta^{\ast}}(S_{t} = S^{\ast}) = \lim_{t \rightarrow \infty} P_{\theta^{\ast}}(S_{t} = S^{\ast}) = 1,
\]
where the last equality is by \eqref{eq: convergence of probability}. Finally, we replace $P_{\theta^{\ast}}(S_{t}= S)$ by $E_{\theta}(\mathbb{I}(S_{t} = S_{t}))$ on the left-hand side and conclude the proof.
\end{proof}

\end{document}